%% file: main.tex
\documentclass[preprint]{article}

\PassOptionsToPackage{numbers, compress}{natbib}

\usepackage[eandd]{neurips_2026}

\usepackage[utf8]{inputenc} 
\usepackage[T1]{fontenc}    
\usepackage{hyperref}       
\usepackage{url}            
\usepackage{booktabs}       
\usepackage{amsfonts}       
\usepackage{nicefrac}       
\usepackage{microtype}      
\usepackage{xcolor}         
\input{math_commands}

\input{packages}

\input{macros}

\title{Certificate-Guided Evaluation of Reinforcement Learning Generalization}

\author{%
Vignesh Subramanian$^{1}$,
\DJ{}or\dj{}e \v{Z}ikeli\'c$^{2}$,
Suguman Bansal$^{1}$\\[0.4em]
$^{1}$School of Computer Science, Georgia Institute of Technology, Atlanta, GA, USA\\
$^{2}$School of Computing and Information Systems, Singapore Management University, Singapore\\[0.4em]
$^{1}$\texttt{\{vignesh,suguman\}@gatech.edu},
$^{2}$\texttt{dzikelic@smu.edu.sg}
}

\begin{document}
\maketitle

\begin{abstract}
This work presents a logic-driven framework to evaluate the performance of reinforcement learning (RL) algorithms in their ability to generalize to unseen tasks. Our framework defines a family of inductive reach-avoid tasks, characterized by structural similarities in task dynamics, enabling evaluation of generalization capabilities. We introduce a neural certificate function that validates trajectories generated by RL algorithms by enforcing key conditions, thereby serving as a litmus test for RL generalization. We empirically demonstrate our method's capability in certifying generalization for several state-of-the-art generalizable RL algorithms on challenging continuous environments. 
Our results show that a lower percentage of certificate function violations correlates with a higher number of test tasks successfully solved, highlighting the effectiveness of our framework in evaluating and distinguishing generalization capabilities of RL algorithms. This work provides a principled approach for benchmarking RL generalization.
\end{abstract}

\input{Section/new_intro}
\input{Section/prelims}

\input{Section/experiments}

\input{Section/conclusion}

\clearpage
\bibliographystyle{bib}
\bibliography{references}

\newpage
\appendix

\input{Appendix/main_appendix}


\end{document}

%% file: math_commands.tex
\usepackage{amsmath,amsfonts,bm}
\usepackage{adjustbox}

\def\eqref#1{equation~\ref{#1}}

\def\1{\bm{1}}

\DeclareMathAlphabet{\mathsfit}{\encodingdefault}{\sfdefault}{m}{sl}
\SetMathAlphabet{\mathsfit}{bold}{\encodingdefault}{\sfdefault}{bx}{n}



%% file: packages.tex
\usepackage{microtype}
\usepackage{graphicx}
\usepackage{booktabs} 

\usepackage{natbib}
\usepackage{multicol}

\usepackage{amsmath}
\usepackage{amssymb}
\usepackage{mathtools}
\usepackage{amsthm}
\usepackage{algorithm}
\usepackage{algorithmic}
\usepackage[capitalize,noabbrev]{cleveref}
\usepackage{url}
\usepackage{mathtools}
\usepackage{float}
\usepackage{placeins}
\usepackage{titletoc}
\usepackage{makecell}

\usepackage{amsfonts}  
\usepackage{nicefrac}      

\usepackage[labelsep=period, labelfont=it, textfont=normal]{caption}
\usepackage{color}
\usepackage{enumerate}
\usepackage{enumitem}
\usepackage{stmaryrd}
\usepackage{stmaryrd}

\usepackage{subcaption}
\usepackage{xspace}

\usepackage{comment}
\usepackage{xcolor}
\usepackage{multirow}
\usepackage{array}
\usepackage{titletoc}

%% file: macros.tex
\newcounter{cond}
\newcommand{\conditem}[1]{\refstepcounter{cond}\item[\textnormal{(\thecond)}]\label{#1}}

\newtheorem{theorem}{Theorem}[section]

\newtheorem{definition}[theorem]{Definition}

 \newcommand{\C}{\mathcal{C}}
 \newcommand{\T}{\mathcal{T}}

\newcommand{\A}{\mathcal{A}}

\newcommand{\N}{\mathbb{N}}

\newcommand{\updategoal}{\mathsf{update}\_\mathsf{goal}}
\newcommand{\updateinit}{\mathsf{update}\_\mathsf{init}}

\newcommand{\supp}{\mathrm{supp}}

\newcommand{\hp}[1]{%
  \ifnum#1>0
    \hphantom{\_}%
    \hp{\the\numexpr#1-1\relax}%
  \fi
}


%% file: Section/new_intro.tex
\section{Introduction}
\label{sec:intro}

The \emph{generalization} problem in reinforcement learning (RL) asks how
agents can learn policies that transfer effectively to new, unseen situations
beyond their training experience, a capability widely recognized as
essential for practical deployment~\citep{MalikLR21, KirkZGR23, Korkmaz24}.
A particularly demanding form of generalization is \emph{zero-shot
generalization}, where agents must transfer to unseen tasks without any
retraining~\citep{beck2023survey, KirkZGR23}. Significant effort has gone
into developing agents capable of zero-shot generalization, spanning
meta-learning~\citep{beck2023survey}, inductive task
structures~\citep{inala2020synthesizing, SubramanianKRB24}, and goal-conditioned
approaches~\citep{naderian2021c}. Yet despite this progress in
\emph{training} zero-shot generalizable agents, the complementary problem of
\emph{evaluating} their generalization ability, i.e., determining in a
principled way whether and to what degree an agent generalizes to unseen tasks, has received surprisingly little attention.

To the best of our knowledge, no principled approach to \emph{evaluating and
comparing generalization properties of different RL agents} has been proposed.
Current practice reduces to brute-force testing: run agents on a large number
of unseen environments and count successes. This lack of a principled
comparison leads to two significant challenges. First, reliable conclusions
require testing on a \emph{large number of new environments}, making
evaluation sample-inefficient and costly. Second, and more fundamentally,
when an agent fails to generalize, current methods offer no means of
identifying \emph{which state-action pairs caused the failure}, leaving
practitioners with no actionable signal for improvement. Ideally, one would
want to flag precisely the behaviors that lead to incorrect generalization,
thereby enhancing explainability of where and why RL generalization breaks
down. Our goal is to address both challenges by proposing a principled
framework for evaluating and comparing generalization properties of RL agents.

We propose to evaluate generalization with respect to \emph{inductive
reach-avoid tasks}~\citep{SubramanianKRB24} --- structured families of tasks
that share the same dynamics but differ inductively in low-level details such
as initial position or goal location. The structural similarity within such a
family provides a natural evaluation substrate: a truly generalizable agent,
trained on a few tasks from the family, should transfer its policy to unseen
members. Fig.~\ref{fig:r4illus} illustrates one such family, where each
task requires navigating a car from an initial to a goal location while
avoiding an obstacle, with both locations shifting left by one unit across
tasks. The key challenge is formalizing what it means for a trajectory to
capture the \emph{essence} of the task family --- a notion precise enough to
serve as an evaluation criterion.

\begin{figure*}[t]
    \centering
    \begin{subfigure}[t]{0.32\textwidth}
        \centering
        \includegraphics[width=\linewidth]{%
            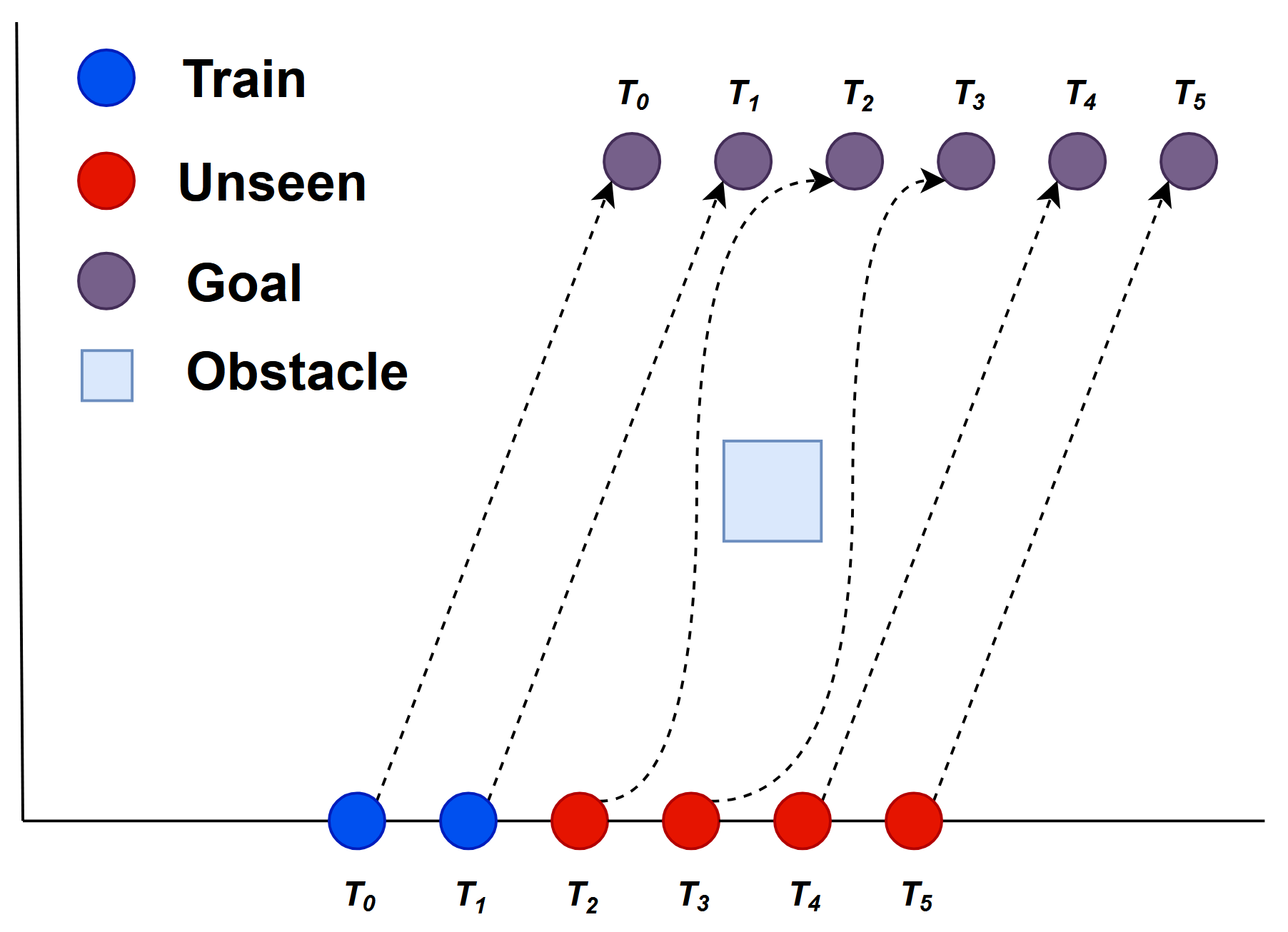}
        \caption{Inductive task}
        \label{fig:r4illus}
    \end{subfigure}
    \hfill
    \begin{subfigure}[t]{0.3\textwidth}
        \centering
        \includegraphics[width=\linewidth]{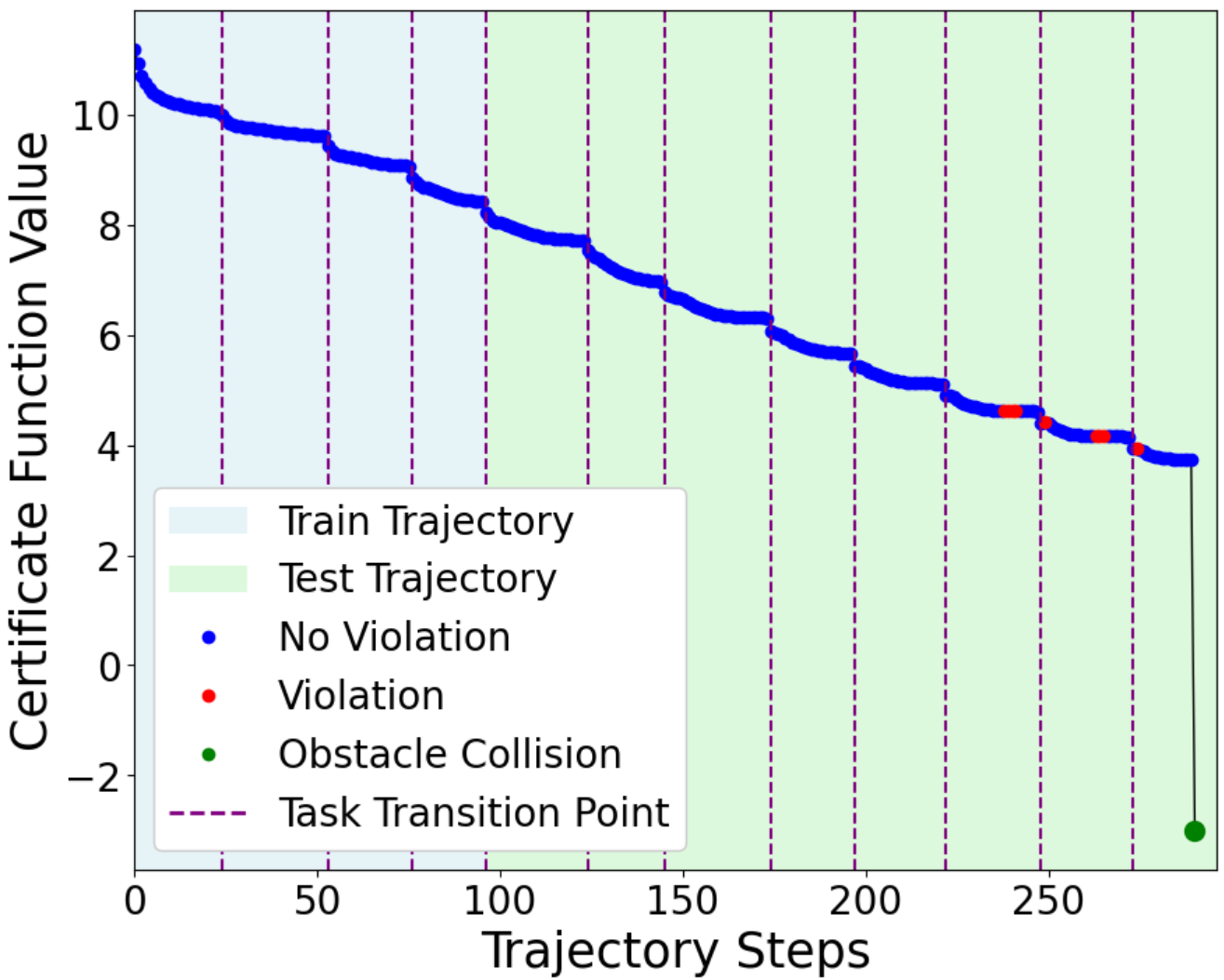}
        \caption{Certificate function values}
        \label{fig:certi-genrl-cpmgp_obs}
    \end{subfigure}
    \hfill
    \begin{subfigure}[t]{0.32\textwidth}
        \centering
        \includegraphics[width=\linewidth]{%
            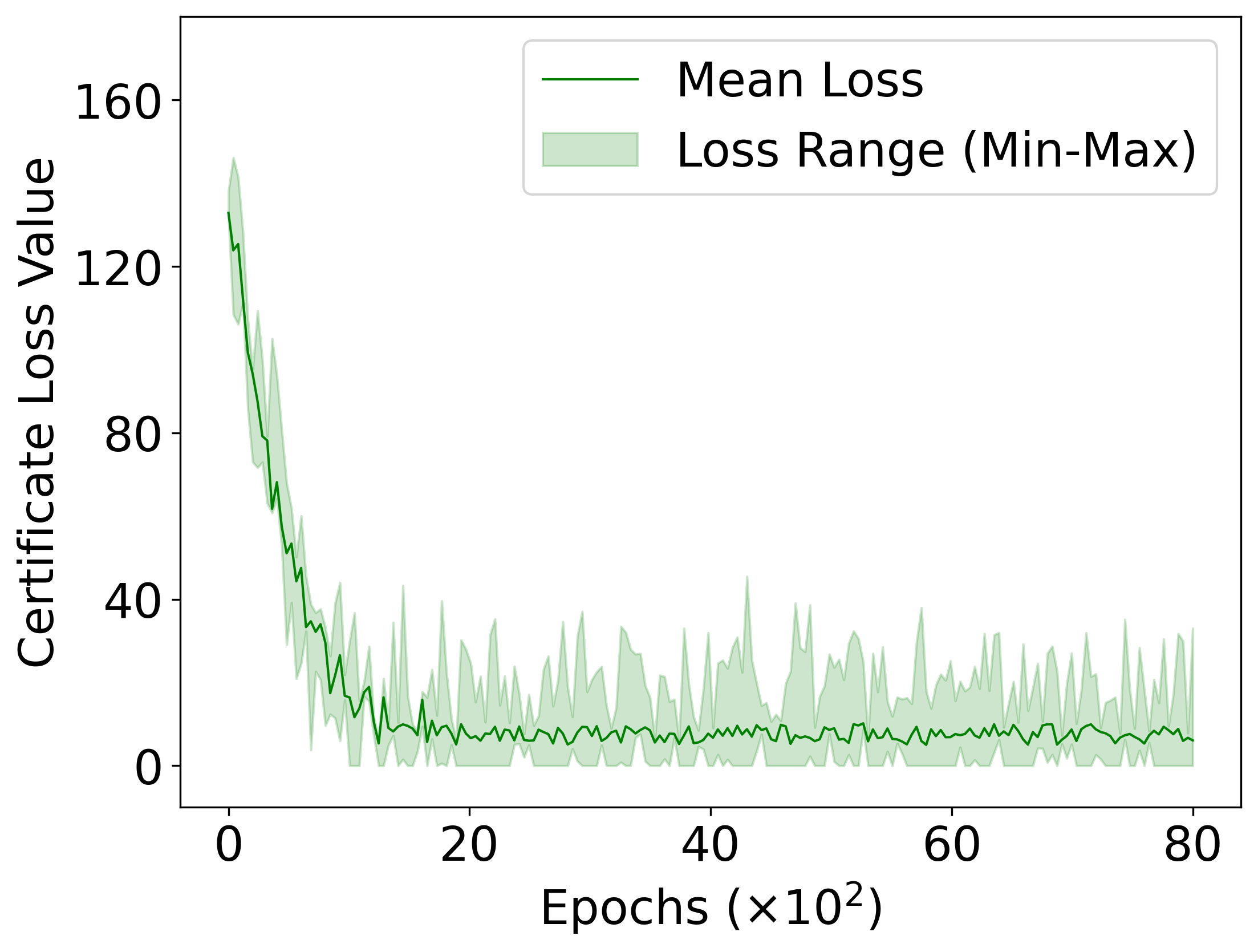}
        \caption{Loss graph for learning the certificate}
        \label{fig:loss_reach_mot}
    \end{subfigure}
    \caption{Car: Moving Initial Point and Goal Point with obstacle:
    The task is to navigate the car from its \textit{initial} position to the corresponding \textit{goal}.}
    \label{fig:motivating_example_cpmgiobs}
\end{figure*}

To this end, we introduce \emph{certificates of correct generalization}. A
certificate is a scalar-valued function $\mathcal{C}$ over MDP states and
task indices, trained on demonstration trajectories from a few tasks in the
inductive family. Intuitively, the certificate value measures an agent's
progress toward satisfying its current reach-avoid task: it is required to be
(1)~non-negative at safe states along demonstration trajectories,
(2)~strictly decreasing along each trajectory until the goal is reached,
(3)~decreasing across subsequent task instances, and (4)~negative at all
unsafe states. Together, these conditions capture the \emph{monotonic
invariant behavior} that a well-generalizing agent should exhibit. We show
(Theorem~\ref{thm:certificate_reach_avoid}) that the existence of such a certificate guarantees
that all demonstration trajectories satisfy their respective reach-avoid
tasks.

This certificate then serves as an evaluation tool. Given a new RL agent, we
run it on unseen tasks from the inductive family and count how many states
along its trajectories \emph{violate} the certificate conditions. Agents with
fewer violations are concluded to generalize better, a conclusion our
experiments consistently validate against ground-truth task-solving
performance across a diverse set of continuous control environments. To
compute certificates in practice, we design a \emph{training procedure} for
learning neural network certificates from oracle demonstration trajectories,
with a loss function enforcing conditions (1)--(4) above. The certificate
function learned for the inductive task in Fig.~\ref{fig:r4illus} is
visualized in Fig.~\ref{fig:certi-genrl-cpmgp_obs}.

\noindent\textbf{Contributions.} We make the following contributions: 
\begin{itemize}[leftmargin=10pt]
    \item We introduce \emph{certificates of correct generalization} for
    inductive reach-avoid tasks, with a formal guarantee that certificates
    witness task satisfaction (Theorem~\ref{thm:certificate_reach_avoid}).
    \item We design an LSTM-based training procedure for learning certificates
    from demonstration trajectories, with a loss function directly enforcing
    the certificate conditions.
    \item We empirically demonstrate across multiple continuous control
    environments that certificate violation rates reliably distinguish RL
    algorithms by their generalization ability, correlating strongly with
    ground-truth task-solving performance.
\end{itemize}

\paragraph{Related Work.}
Generalizable RL has been approached primarily through meta-learning, where
agents learn to adapt across task distributions. Gradient-based methods such
as MAML~\citep{finn2017model} and its variants~\citep{nagabandi2019deep}
enable rapid adaptation, while VariBAD~\citep{zintgraf2021varibad} uses
variational inference for zero-shot generalization via task embeddings. Most
relevant to our setting, PSMP~\citep{inala2020synthesizing} and
GenRL~\citep{SubramanianKRB24} exploit inductive task structure to improve
transferability. However, all of these works focus on \emph{training}
generalizable agents; none proposes a principled method for \emph{evaluating}
or \emph{comparing} generalization properties across agents.

Certificate functions have recently gained prominence as tools for verifying
that control policies satisfy given specifications~\citep{DawsonGF23}. Neural
network certificates have been studied for reachability, safety, and stability
in both deterministic~\citep{ChangRG19, AbateAGP21, EdwardsPA24, Zhang0V023}
and stochastic~\citep{LechnerZCH22, ZikelicLHC23, MathiesenCL23,
ChatterjeeHLZ23} settings. However, all existing methods certify a
\emph{single policy on a single task}. To our knowledge, we are the first to
repurpose certificates for cross-agent, cross-task evaluation of RL
generalizability.

%% file: Section/prelims.tex
\section{Preliminaries}
\label{sec:prelims}

\paragraph{Markov Decision Process.} RL environments are formally modeled via Markov decision processes. A {\em Markov decision process (MDP)} $\mathcal{M}$ is a tuple $(S, A, P)$, where $S$ is a set of states, $A$ is a set of actions, and $P: S \times A \times S \to \mathbb{R}_{\geq 0}$ is a probabilistic transition function. We use $P(\cdot \mid s,a)$ to denote the probability distribution of the successor state after taking action $a$ in state $s$. In this work, we restrict our attention to {\em deterministic MDPs}, meaning that each $P(\cdot \mid s,a)$ is a Dirac distribution assigning probability mass $1$ to a single state $s'$ and probability mass $0$ to every other state. By a slight abuse of notation, we write $s' = P(s,a)$.

A {\em trajectory} $\zeta$ in an MDP is an infinite sequence $(s_t,a_t)_{t=0}^\infty$ of state-action pairs. A {\em finite trajectory} is a finite prefix of a trajectory. A (pure positional) {\em policy} in an MDP is a map $\pi: S \rightarrow A$, which to each state assigns an action to be taken. For simplicity, we denote the trajectory obtained by a pure positional policy $\pi$ by $(s_t)_{t=0}^\infty$  where $s_{j+1} = P(s_j, \pi(s_j))$.

\paragraph{Reach-Avoid Tasks.} This work focuses on {\em reach-avoid} tasks.  Given an MDP with states $S$, a reach-avoid task is defined by the tuple $ (G, \eta, X_u) $ comprising {\em goal states} $G \subseteq S$, {\em initial state distribution} $\eta$ over MDP states $S$, and {\em unsafe states} $X_u \subseteq S\setminus \supp(\eta)$ where $\supp(\eta)$ is the support of the initial state distribution. 
The objective of a reach-avoid task is to reach the goal states from the initial states while avoiding all unsafe states along the trajectory. Formally, a trajectory $ \zeta = (s_t, a_t)_{t=0}^\infty$ satisfies a reach-avoid task $(G, \eta, X_u)$, denoted $ \zeta \models (G, \eta, X_u) $, if there exists $T \geq 0$ such that $ s_0 \sim \eta $, $ s_T \in G $, and $s_j \in S\setminus X_u$ for all $0 \leq j \leq T$.
For simplicity, we use the notation $X_s = S \setminus X_u$ to denote the {\em safe states}.

\paragraph{Assumption.}  We assume that all states in the set of unsafe states $X_u$ are sink states. This means that once a trajectory enters any state $s \in X_u$, it cannot leave, i.e., $P(s, a) = s$ for all $s \in X_u$ and $a \in A$. This assumption reflects scenarios where entering an unsafe state represents an irreversible failure, such as collisions in robotic navigation. As a result, any trajectory that enters $X_u$ is immediately terminated and considered invalid for the reach-avoid task.

\section{Evaluation of RL Generalizability}\label{sec:framework}

We present our framework for evaluating and comparing generalization properties of RL agents. 

Our framework evaluates the generalization of RL algorithms based on their ability to extrapolate to similar but different tasks. 
To achieve this, we consider generalization with respect to {\em inductive tasks} \citep{SubramanianKRB24}, which form a family of tasks that are structurally similar but differ inductively in the low-level details. We consider {\em inductive reach-avoid tasks}. We are guided by the intuition that these tasks are so similar that their policies should also be similar, making them a good fit for evaluating RL generalizability. Inductive reach-avoid tasks are formally defined in Section~\ref{sec:reachavoidtasks}.

In order to evaluate generalization properties of RL algorithms, our framework uses a set of task-trajectory pairs that serve as {\em demonstrations} of what a trajectory that satisfies some reach-avoid task looks like. It then uses these demonstrations to extract a quantitative measure of {\em progress} along demonstration trajectories towards satisfying their reach-avoid tasks. This progress measure is formally captured via {\em certificates of correct generalization}, a notion that we introduce in Section~\ref{sec:certificate}.

Finally, our framework is presented in Section~\ref{sec:evaluation}. Our framework first trains a certificate of correct generalization from a given set of task-trajectory demonstration pairs. It then runs each RL agent on a set of new and unseen tasks, and uses the certificate to quantify and compare the generalization capabilities of each RL agent.

\subsection{Inductive Reach-Avoid Tasks}\label{sec:reachavoidtasks}

An inductive reach-avoid task is a sequence of reach-avoid tasks over the same MDP, such that the subsequent reach-avoid task builds upon the previous one by progressively updating the goal region, the initial state distribution, or both. We assume that the set of unsafe states is the same for all tasks in the sequence, as these correspond to e.g.,~physical obstacles or collisions. For a set $S$, denote by $\mathcal{P}(S)$ the set of all subsets of $S$ and by $\mathcal{D}(S)$ the set of all probability distributions over $S$.

\begin{definition}[Inductive reach-avoid task]\label{def:task}
Consider an MDP with a set of states $S$. An \emph{inductive reach-avoid  task} is given by a tuple $ \mathcal{T} = (\mathcal{T}_0, \updategoal, \updateinit) $, where
 (1)~$ \mathcal{T}_0 = (G_0, \eta_0, X_u) $ is the \emph{base task} with initial goal states $ G_0  \subseteq S $, initial state distribution $ \eta_0 $ and unsafe states $X_u$,
(2)~$ \updategoal: \mathcal{P}(S) \mapsto \mathcal{P}(S) $ is the \emph{goal update function}, and  
(3)~$ \updateinit: \mathcal{D}(S) \mapsto \mathcal{D}(S) $ is the \emph{initial distribution update function}.
Then, the inductive task $\mathcal{T}$ denotes the sequence of reach-avoid tasks $\T_i = (G_i, \eta_i, X_u)$ for each $i\in\mathbb{N}_0$, where $\T_0$ is defined as above  and $\T_{i+1} = (\updategoal(G_i), \updateinit(\eta_i), X_u)$  for  $i\in\mathbb{N}_0$.
\end{definition}

For each task to be satisfiable, we assume that the goal region and the support of the initial distribution of each task are contained within the safe set $X_s$, i.e.~that $ G_i \subseteq X_s$ and $ \mathrm{supp}(\eta_i) \subseteq X_s$ hold for all $i \in \mathbb{N}_0$.

\subsection{Certificates of Correct Generalization}\label{sec:certificate}

We now define our certificates for validating generalizability in inductive reach-avoid tasks. Consider an inductive reach-avoid task $\T$ giving rise to a sequence of reach-avoid tasks $\T_0, \T_1, \ldots$. Our certificate is defined with respect to a set $\xi = \{ (\T_0, \zeta_0), (\T_1, \zeta_1), \ldots \}$ of task-trajectory pairs, where each $\zeta_i$ is a finite (possibly empty) set of MDP trajectories. These trajectories serve as {\em demonstrations} of what a trajectory that satisfies the reach-avoid task $\T_i$ should look like.

Before defining the certificate, we fix some additional notation. For each task index $i \in \mathbb{N}_0$, let $\zeta_i = \{\zeta_i^1,\dots,\zeta_i^{K_i}$\} be the set of demonstration trajectories for task $\T_i$ with $K_i \in \mathbb{N}_0$ and each $\zeta_i^k = (s_{i,j}^k)_{j=0}^\infty$ being an MDP trajectory. We use $T_{G_i}(\zeta_i^k) = \inf\{j \in \mathbb{N}_0 \mid s_{i,j}^k \in G_i\}$ to denote the index of the first state along trajectory $\zeta_i^k$ that is in the goal region of the task $\T_i$, with $T_i = \infty$ if the goal region is not reached.

Formally, a certificate of correct generalization is a function $\mathcal{C}: S \times \mathbb{N}_0 \to \mathbb{R}$ assigning a real value to each MDP state and task index pair. The certificate value is required to be (1)~non-negative at all safe states along the demonstration trajectories, (2)~strictly decreasing along each demonstration trajectory until it reaches a goal state, (3)~decreasing across subsequent tasks, and (4)~negative at all unsafe states.

\begin{definition}[Certificates of correct generalization]\label{def:certificate_reach_avoid}
Consider an inductive reach-avoid task $\T$ and a set $\xi = \{ (\T_0, \zeta_0), (\T_1, \zeta_1), \ldots \}$ of task-trajectory pairs defined as above. A {\em certificate of correct generalization} for the set of demonstration trajectories $\xi$ is a function $\mathcal{C}: S \times \mathbb{N}_0 \to \mathbb{R}$, such that the following conditions hold:
\begin{description}
    \conditem{reach-avoid-def:nonneg}(Non-Negativity at Safe States) For each task index $i \in \mathbb{N}_0$, trajectory $\zeta_i^k$ and safe state $s_{i,j}^k \in X_s$ along the trajectory $\zeta_i^k$, we have $\mathcal{C}(s_{i,j}^k, i) \geq 0$.
    \conditem{reach-avoid-def:sdwt}(Strict Decrease Within a Task Until the Goal) For each task index $i \in \mathbb{N}_0$, there exists $\varepsilon_i > 0$ such that for each trajectory $\zeta_i^k$ and safe state $s_{i,j}^k \in X_s$ along the trajectory $\zeta_i^k$ with $0 \leq j < T(\zeta_i^k)$, we have $\mathcal{C}(s_{i,j+1}^k, i) \geq 0$ and 
    $\mathcal{C}(s_{i,j}^k, i) > \mathcal{C}(s_{i,j+1}^k, i) + \varepsilon_i$.
    \conditem{reach-avoid-def:sdat}(Decrease Across Subsequent Tasks) For each task index $i \in \mathbb{N}_0$, the certificate value at the goal state of every trajectory for task $\T_i$ is greater than the value at the initial state of every trajectory for task $\T_{i+1}$, i.e.
    \[
    \min_{\substack{1 \leq k \leq K_i\\ T(\zeta_i^k)<\infty}}
    \mathcal{C}(s^k_{i,T(\zeta_i^k)},i)
    >
    \max_{\substack{1 \leq k \leq K_{i+1}\\ T(\zeta_{i+1}^k)<\infty}}
    \mathcal{C}(s^k_{i+1,T(\zeta_{i+1}^k)},i+1)
    \]
    \conditem{reach-avoid-def:neg}(Strict Negativity at Unsafe States) For all unsafe states $s \in X_u$ and task indices $i \in \mathbb{N}_0$, we have $\mathcal{C}(s, i) < 0$. 
\end{description}
\end{definition}

The intuition behind our certificate is that it captures and quantifies {\em progress} along demonstration trajectories towards satisfying their reach-avoid tasks. This is because the certificate value is required to be non-negative at safe states and to strictly decrease along a demonstration trajectory while keeping the certificate value non-negative, until the goal region is reached (Conditions 1 and 2). Hence, the trajectory must eventually reach the goal region, and so the certificate value measures progress towards the goal. Moreover, as the certificate value is negative at unsafe states (condition 4), the trajectory must avoid unsafe states until the goal is reached, and so the reach-avoid task is satisfied. Finally, the certificate furthermore captures the inductive relation between subsequent task instances in the inductive reach-avoid task (condition 3).

The following theorem formalizes this intuition, and shows that if a certificate of correct generalization exists, then all demonstration trajectories satisfy their reach-avoid tasks. Hence, our certificate indeed captures and quantifies progress along demonstration trajectories towards satisfying their reach-avoid tasks. This will allow us to utilize our certificates to evaluate generalization properties of RL algorithms, by checking whether they preserve this progress behaviour along the trajectories generated for new and unseen tasks.



\begin{theorem}\label{thm:certificate_reach_avoid}
Let $\T$ be an inductive reach-avoid task giving rise to a sequence of reach-avoid tasks $\T_0, \T_1, \cdots$, and let $\xi = \{ (\T_0, \zeta_0), (\T_1, \zeta_1), \ldots \}$ be a set of task-trajectory pairs.
If a certificate of correct generalization for the set of demonstration trajectories $\xi$ exists, then $\zeta_i^k \models \T_i$  for all task indices $i \in \mathbb{N}_0$ and trajectory indices $1 \leq k \leq K_i$.
\end{theorem}

\begin{proof}
    Suppose that $\mathcal{C}$ is a certificate of correct generalization for $\xi$. We need to show that, for each task index $i \in \mathbb{N}_0$ and trajectory index $1 \leq k \leq K_i$, we have $\zeta_i^k \models \T_i$. That is, if $\zeta_i^k = (s_{i,j}^k)_{j=0}^\infty$ is an MDP trajectory, we need to show that there exists an index $T \in \mathbb{N}_0$ such that $s_{i,T}^k \in G_i$ and $s_{i,j}^k \not\in X_u$ for $0 \leq j <T$.

    The existence of an index $T \in \mathbb{N}_0$ such that $s_{i,T}^k \in G_i$ follows by conditions $1$ and $2$ in Definition~\ref{def:certificate_reach_avoid}. This is because the support of the initial state distribution $\eta_i$ of task $\T_i$ is assumed to be contained within the safe set $X_s$ (see Section~\ref{sec:reachavoidtasks}). Hence, by condition 1, we have $\mathcal{C}(s_{i,0}^k,i) \geq 0$. On the other hand, by condition 2, it follows that until a state in $G_i$ is reached, we will have $\mathcal{C}(s_{i,j}^k,i) \geq \mathcal{C}(s_{i,j+1}^k,i) + \varepsilon_i$ and $\mathcal{C}(s_{i,j+1}^k,i) \geq 0$. Hence, a state in $G_i$ is reached in at most $T \leq \lceil \mathcal{C}(s_{i,0}^k,i) / \varepsilon_i \rceil$ steps. Moreover, since $\mathcal{C}(s_{i,j}^k,i) \geq 0$ holds for $0 \leq j \leq T$, it follows by condition 4 that $s_{i,j}^k \not\in X_u$ for $0 \leq j <T$. This concludes the proof that $\zeta_i^k \models \T_i$.
\end{proof}

\subsection{Evaluating RL Generalizability via Certificates}\label{sec:evaluation}

We now show how our certificates of correct generalization can serve as a tool for evaluating and comparing the generalization capabilities of different RL algorithms $\pi_1, \dots, \pi_n$.

First, our framework learns a certificate of correct generalization for a given set of demonstration trajectories $\xi = \{ (\T_0, \zeta_0), (\T_1, \zeta_1), \ldots \}$. The training procedure for learning certificates and the details on how demonstration trajectories are collected are presented in Section~\ref{sec:learning}.

Second, a finite number of test reach-avoid tasks $\Phi^{\text{test}} \subseteq \T$ are chosen to assess generalizability of given RL agents $\pi_1, \dots, \pi_n$. We require that demonstration trajectories for the test reach-avoid tasks are not used in the certificate training procedure. That is, a reach-avoid task $\T_i$ can be used as a test task only if the set of demonstration trajectories $\zeta_i$ is empty.

Finally, for each RL agent $\pi_i$ and for each test reach-avoid task $\phi \in \Phi^{\text{test}}$, we use $\pi_i$ to generate a trajectory $\xi_{\pi_i}^\phi$ for the test task $\phi$. We then count the total number of violations of the certificate conditions along $\xi_{\pi_i}^\phi$. That is, for each of Conditions~1-4 in Definition~\ref{def:certificate_reach_avoid}, we count at how many states along the trajectory $\xi_{\pi_i}^\phi$ the condition is violated. In reach-avoid tasks, all steps beyond the first unsafe state (denoted as "OC") are treated as violations. We denote by $N_{\pi_i}^\phi$ the total number of certificate violations for task $\phi$, and define $ N_{\pi_i} = \sum_{\phi \in \Phi^{\text{test}}} N_{\pi_i}^\phi $ to be the {\em total number of certificate violations} of the RL agent $\pi_i$ across all test tasks in $\Phi^{\text{test}}$. The {\em percentage of certificate violations} is computed using the ratio of the number of certificate violations to the total number of states in the trajectory. We then compare the percentages of certificate violations for different RL agents $\pi_1,\dots,\pi_n$, and conclude that RL agents $\pi_i$ with lower percentages of certificate violations have better generalizability properties.

This framework provides a principled method to evaluate generalization properties of RL agents on unseen tasks, by using certificates of correct generalization from given demonstrations as a validation tool. By comparing the percentages of certificate violations across different RL algorithms, we gain insights into their ability to generalize to new tasks and preserve the progress behavior captured by the certificate.

\section{Learning Certificates}\label{sec:learning}

This section describes the training procedure for learning a certificate of correct generalization. Consider an inductive reach-avoid task $\T$ giving rise to a sequence of reach-avoid tasks $\T_0, \T_1, \ldots$. Given a set of demonstration trajectories $\xi = \{ (\T_0, \zeta_0), (\T_1, \zeta_1), \ldots \}$ defined as in the previous section, our objective is to train a neural network to approximate the certificate function $\C: S\times \N_0 \to \mathbb{R} $ defined in Definition~\ref{def:certificate_reach_avoid}.

\noindent\textbf{Training Data.} We start by describing how our training data, i.e.~the demonstration trajectories used to train a neural network certificate, are collected. We train the certificate on finite-length prefixes of trajectories for tasks $ \T_0, \ldots, \T_{n-1}$, i.e., the first $n$ tasks in the inductive reach-avoid task $\T$, where $n \in \mathbb{N}$ is an algorithm parameter. We use finite-length prefixes because one cannot obtain infinite-length trajectories in practice. We refer to tasks $\T_0, \ldots, \T_{n-1}$ as the {\em training tasks} and tasks $\T_n, \T_{n+1}, \ldots$ as the {\em testing tasks}. Hence, demonstration trajectories will be finite-length trajectories in the sets $\zeta_i$ for $0 \leq i \leq n-1$, whereas $\zeta_i = \emptyset$ for $i \geq n$.

For each training task $\T_i$ with $0 \leq i \leq n-1$, the set of demonstration trajectories $\zeta_i$ comprises (1)~positive examples, i.e. trajectories that satisfy the reach-avoid task $\T_i$, and (2)~negative examples, i.e. trajectories that violate $\T_i$. Positive examples are needed in order to capture progress behaviour in trajectories that satisfy the reach-avoid task with respect to Conditions 1-3 in Definition~\ref{def:certificate_reach_avoid}. Among the negative examples, all trajectories are of the kind that reach an unsafe state and thus violate the avoidance objective, which are needed to capture condition 4 in Definition~\ref{def:certificate_reach_avoid}. We do not consider finite-length trajectories that do not violate safety but violate reachability, since one cannot ascertain whether such a trajectory could have satisfied the specification if it were longer.  

These training trajectories are collected using an oracle that generates trajectories of positive and negative examples in the environment for the given reach-avoid task. Positive and negative trajectories are obtained by sampling many trajectories in the environment using a planner, and labeling them as positive or negative examples depending on whether the trajectory satisfies the reach-avoid task.

\noindent\textbf{Training Procedure.} Our training procedure is formulated as a learning task, where the goal is to minimize a loss function that is designed to enforce Conditions 1-4 in Definition~\ref{def:certificate_reach_avoid} along demonstration trajectories. In particular, our loss function is a sum of four loss terms, each corresponding to one condition in Definition~\ref{def:certificate_reach_avoid}, where the outer sum is taken over all training tasks and all demonstration trajectories:

\begin{align*}
&\mathsf{Loss}
= \sum_{i=0}^{n-1} \sum_{(s_{i,j})_{j=0}^{l} \in \zeta_i} \Bigl(
\underbrace{
\lambda_{\text{safe}} \sum_{j=0}^{l}
\max\{0,-C(s_{i,j},i)\}\,1_{s_{i,j}\in X_s}
}_{\scriptsize\text{cond.~(\ref{reach-avoid-def:nonneg})}}
+
\underbrace{
\lambda_{\text{unsafe}} \sum_{j=0}^{l}
\max\{0,C(s_{i,j},i)\}\,1_{s_{i,j}\in X_u}
}_{\scriptsize\text{cond.~(\ref{reach-avoid-def:neg})}}
\\
&\quad+
\underbrace{
\sum_{j=0}^{l-1}
\max\{0,C(s_{i,j+1},i)+\varepsilon_i-C(s_{i,j},i)\}\,
1_{(s_{i,j},s_{i,j+1})\in X_s}
}_{\scriptsize\text{cond.~(\ref{reach-avoid-def:sdwt})}}
\\
&\quad+
\underbrace{
\sum_{(s_{i+1,j})_{j=0}^{l}\in\zeta_{i+1}}
\max\{0,C(s_{i+1,0},i+1)+\varepsilon_i-C(s_{i,l},i)\}\,
1_{s_{i,l}\in G_i}
}_{\scriptsize\text{cond.~(\ref{reach-avoid-def:sdat})}}
\Bigr).
\end{align*}

Here, for each task index $i \in \mathbb{N}_0$ and for each finite-length trajectory in $\zeta_i$, we use $l$ to denote the index of the first visit to either the goal region $G_i$ of $T_i$ (for positive example trajectories) or the unsafe set $X_u$ (for negative example trajectories). The first summation enforces the non-negativity condition (\ref{reach-avoid-def:nonneg}) at safe states $X_s$ by penalizing violations. The fourth summation enforces the strict negativity condition (\ref{reach-avoid-def:neg}) at unsafe states $X_u$. The scaling factors $\lambda_{\text{safe}}$ and $\lambda_{\text{unsafe}}$ are hyperparameters that control the importance of these penalties. The second summation penalizes violations of monotonic decrease within task trajectories, enforcing the certificate function to decrease as the trajectory progresses. The third summation penalizes violations of monotonic decrease across tasks, ensuring the certificate values at the goal states of task $\mathcal{T}_i$ are greater than those of the start states of $\mathcal{T}_{i+1}$.

We use the \emph{Long Short-Term Memory}~\citep{hochreiter1997long} (LSTM) architecture to train the certificate function. The LSTM acts as a function $ \mathcal{C}(s, i) $ that takes a state $s$ and a task index $i$ as input and outputs a real number representing the certificate value. 
The LSTM is trained on the set of all sequences of concatenated trajectories of the form $ (\zeta_0^{k_0}, \zeta_1^{k_1}, \ldots, \zeta_{n-1}^{k_{n-1}}) $, where $ \zeta_i^{k_i} \in \zeta_i$ is a demonstration trajectory for task $ \mathcal{T}_i $ (generated by the oracle). We generate and use multiple batches of $\xi$ for effective training. We choose LSTMs since they are  well-suited to capture sequential dependencies in data. 
Trajectories represent ordered sequences of states, where the progression along each trajectory encodes information about task dynamics, safety, and proximity to the target state. 
Furthermore, concatenated trajectories ensure that dynamics across subsequent tasks are encoded. Since LSTMs are designed to maintain memory and capture long-term dependencies, we expect the LSTMs to capture these intra- and inter-trajectory dependencies while respecting safety conditions.

%% file: Section/experiments.tex
\begin{figure*}[t]
    \centering
    \begin{subfigure}[t]{0.30\textwidth}
        \centering
        \includegraphics[width=\textwidth]{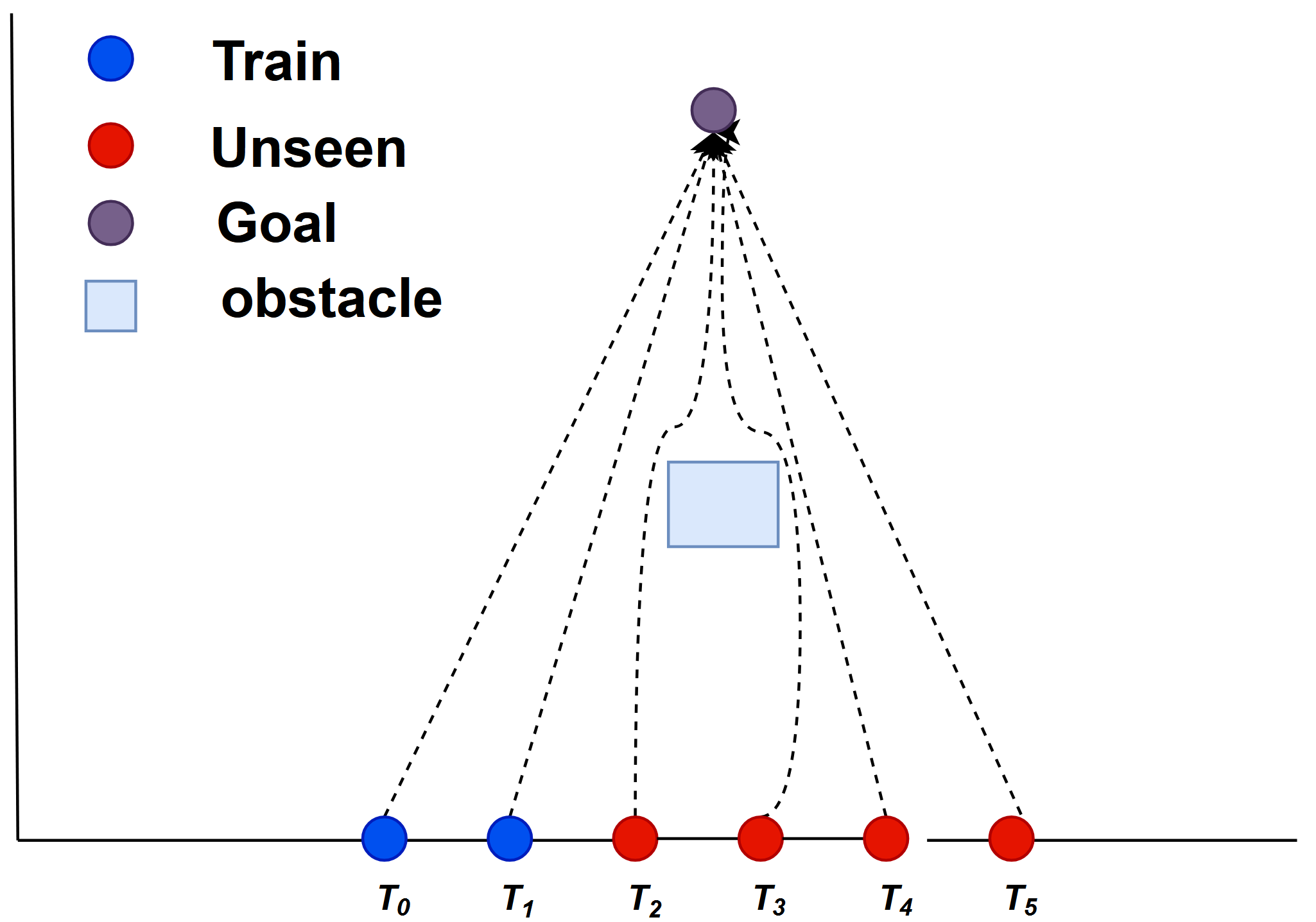}
        \caption{Car: Moving initial point and stationary goal point with obstacle}
        \label{fig:cpmiobsillus}
    \end{subfigure}\hfill
    \begin{subfigure}[t]{0.30\textwidth}
        \centering
        \includegraphics[width=\textwidth]{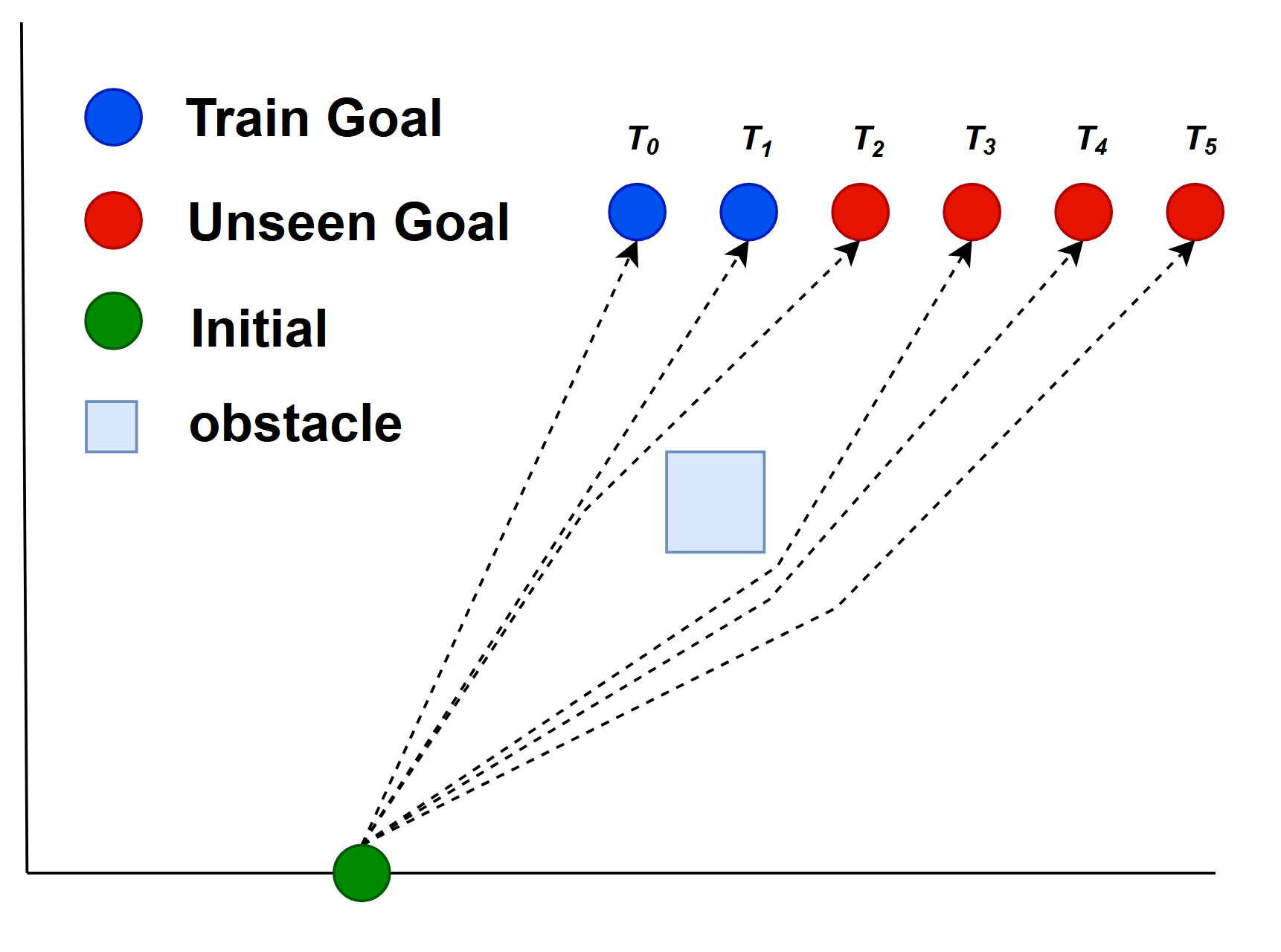}
        \caption{Car: Stationary initial point and moving goal point with obstacle}
        \label{fig:cpmgobsillus}
    \end{subfigure}\hfill
   \begin{subfigure}[t]{0.25\textwidth}
       \centering
       \includegraphics[width=\textwidth]{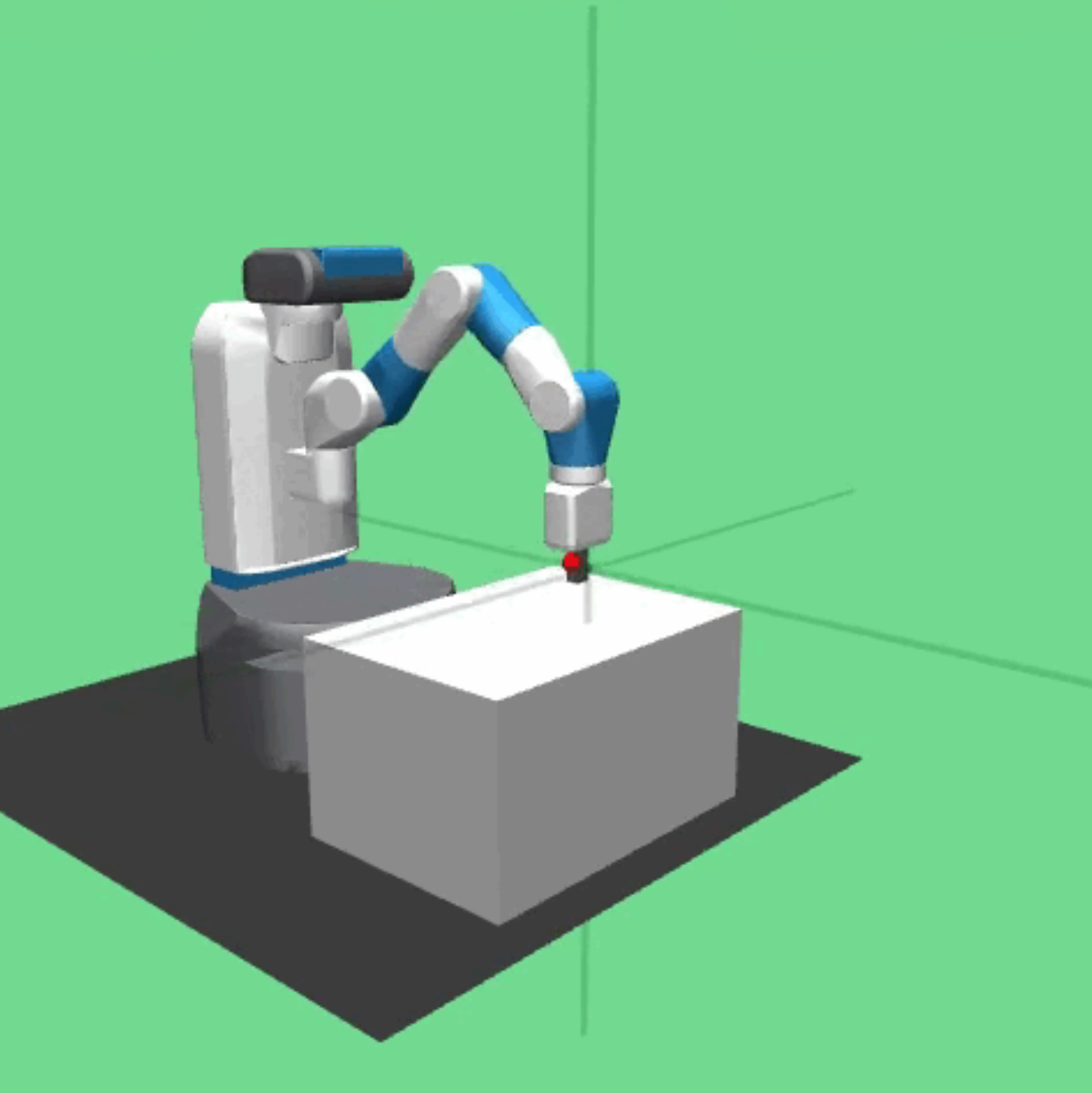}
       \caption{Fetch-Reach - Moving initial point and stationary goal point without obstacle}
       \label{fig:fetchillus}
   \end{subfigure}
    \caption{Illustrations of tasks in the Car environment and the Fetch-Reach environment.}
    \label{fig:side_by_side_figures}
\end{figure*}

\begin{figure*}[t]
    \centering

    \begin{subfigure}[t]{0.49\textwidth}
        \centering
        \includegraphics[width=\textwidth]{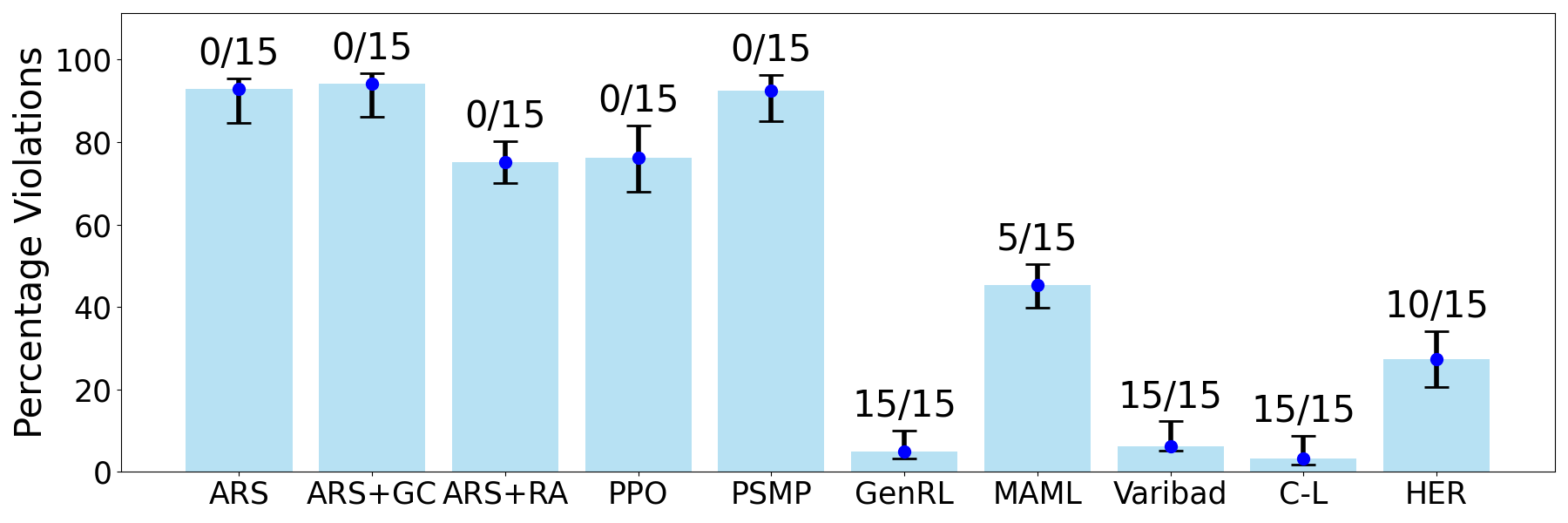}
        \caption{Moving Initial Point (No Obstacle)}
        \label{fig:initial_point}
    \end{subfigure}
    \hfill
    \begin{subfigure}[t]{0.49\textwidth}
        \centering
        \includegraphics[width=\textwidth]{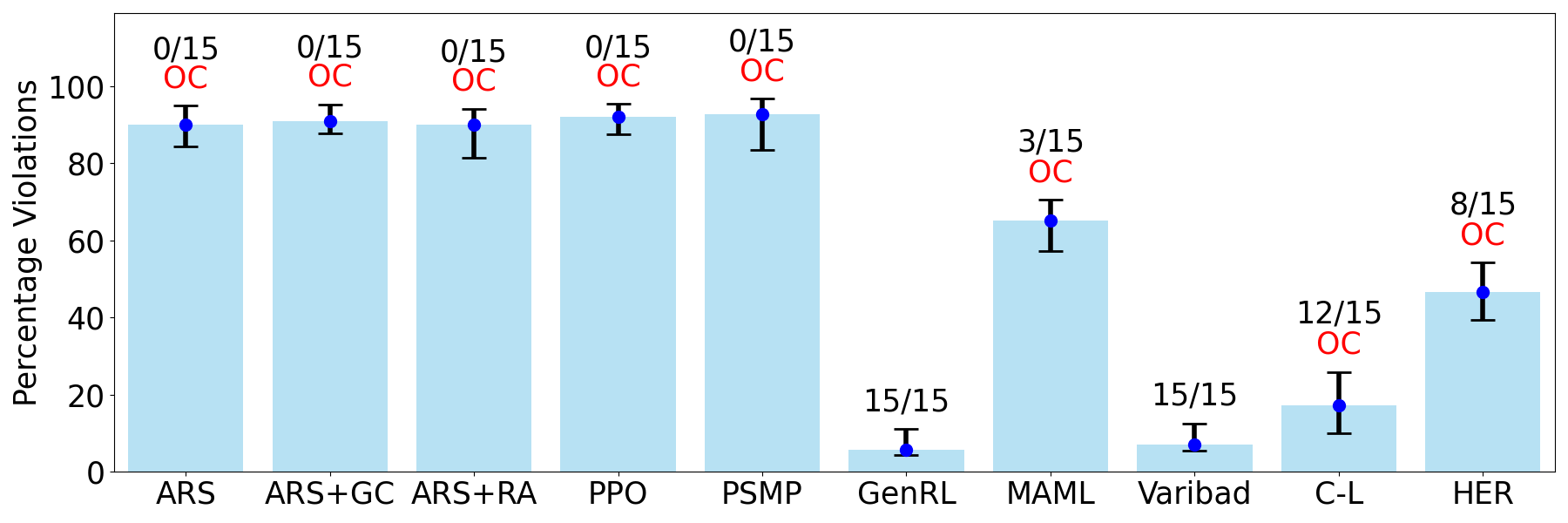}
        \caption{Moving Initial Point (With Obstacle)}
        \label{fig:initial_point_obs}
    \end{subfigure}
 
    \vspace{0.5em}
    
    \begin{subfigure}[t]{0.49\textwidth}
        \centering
        \includegraphics[width=\textwidth]{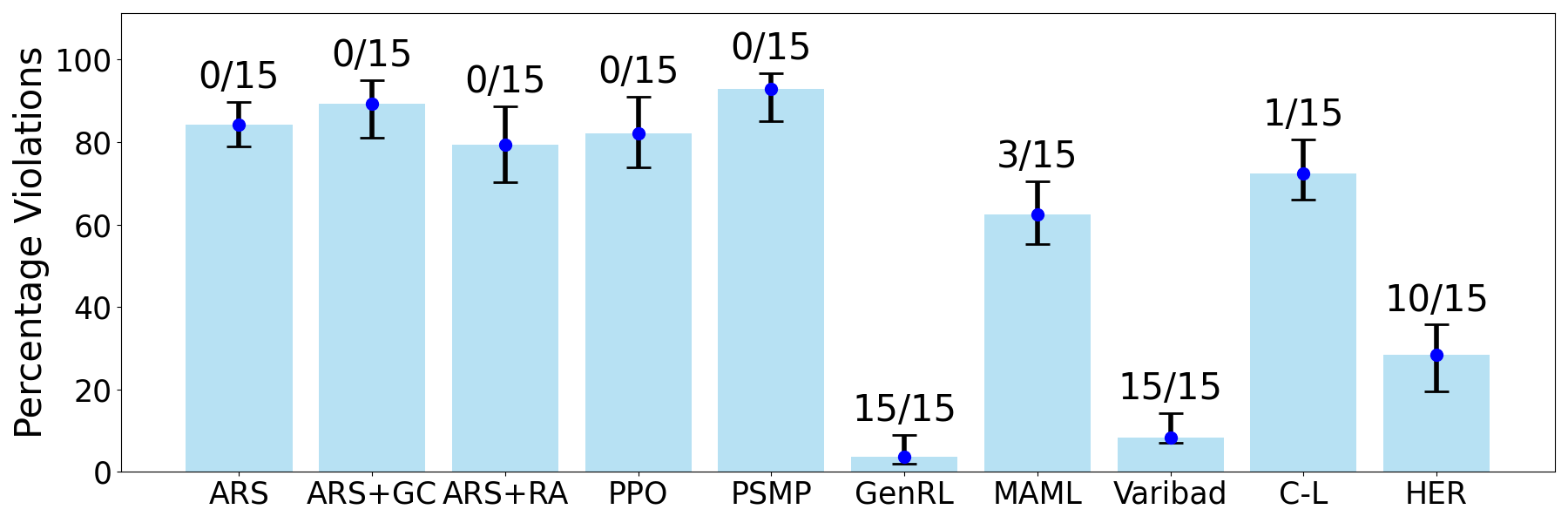}
        \caption{Moving Goal Point (No Obstacle)}
        \label{fig:goal_point}
    \end{subfigure}
    \hfill
    \begin{subfigure}[t]{0.49\textwidth}
        \centering
        \includegraphics[width=\textwidth]{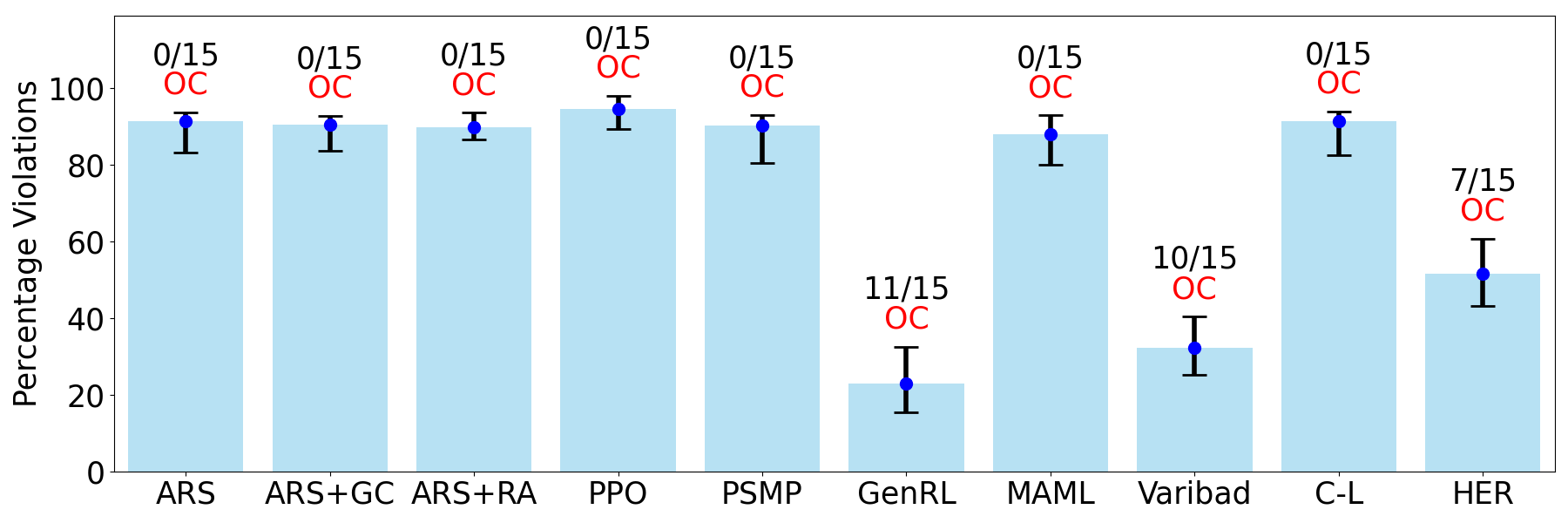}
        \caption{Moving Goal Point (With Obstacle)}
        \label{fig:goal_point_obs}
    \end{subfigure}

    \vspace{0.5em}
    
    \begin{subfigure}[t]{0.49\textwidth}
        \centering
        \includegraphics[width=\textwidth]{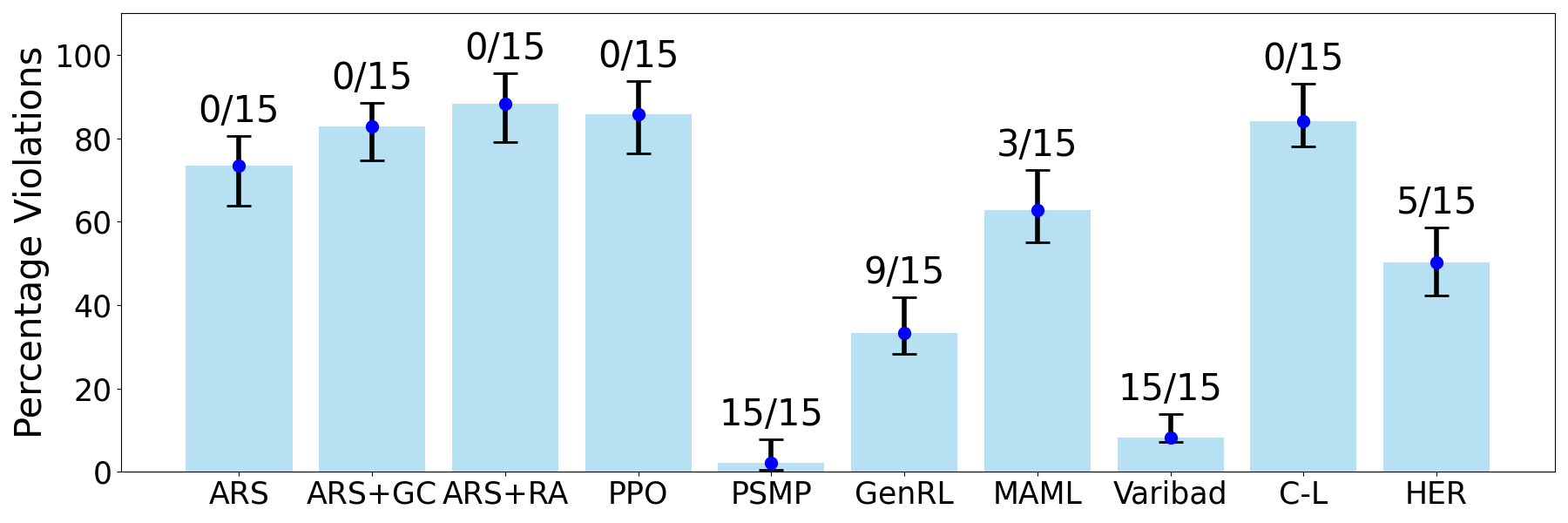}
        \caption{Moving Initial and Goal Point (No Obstacle)}
        \label{fig:initialandgoal_point}
    \end{subfigure}
    \hfill
    \begin{subfigure}[t]{0.49\textwidth}
        \centering
        \includegraphics[width=\textwidth]{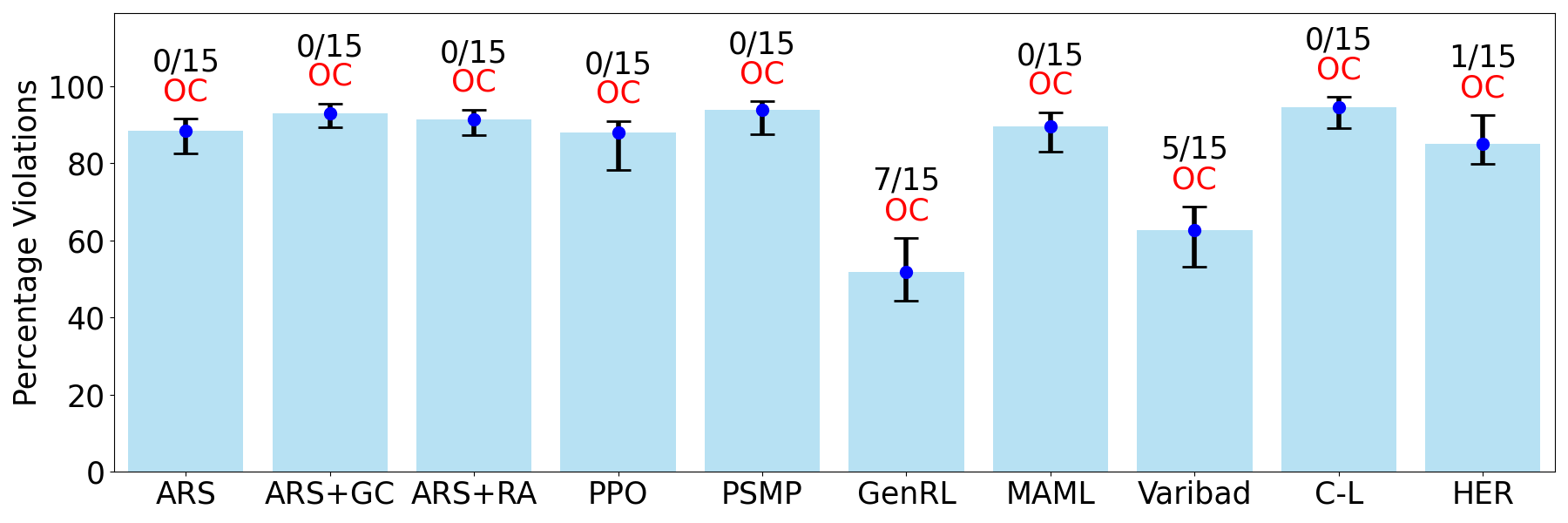}
        \caption{Moving Initial and Goal Point (With Obstacle)}
        \label{fig:initialandgoal_point_obs}
    \end{subfigure}
    
    \caption{\textbf{Violation plots for the Car environment.}}
    \label{fig:car_comparison_histograms}
\end{figure*}

\begin{figure*}[t]
    \centering
    \begin{subfigure}[t]{0.4\textwidth}
        \centering
        \includegraphics[width=\textwidth]{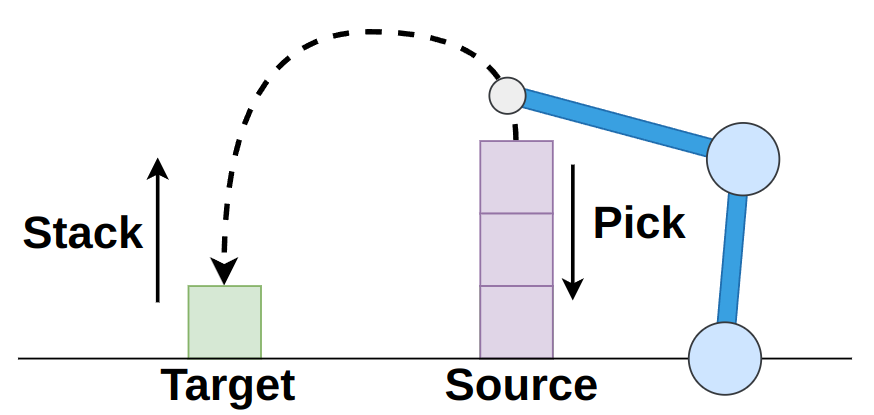}
        \caption{}
        \label{fig:reacher_illus}
    \end{subfigure}
    \hfill
    \begin{subfigure}[t]{0.55\textwidth}
        \centering
        \includegraphics[width=\textwidth]{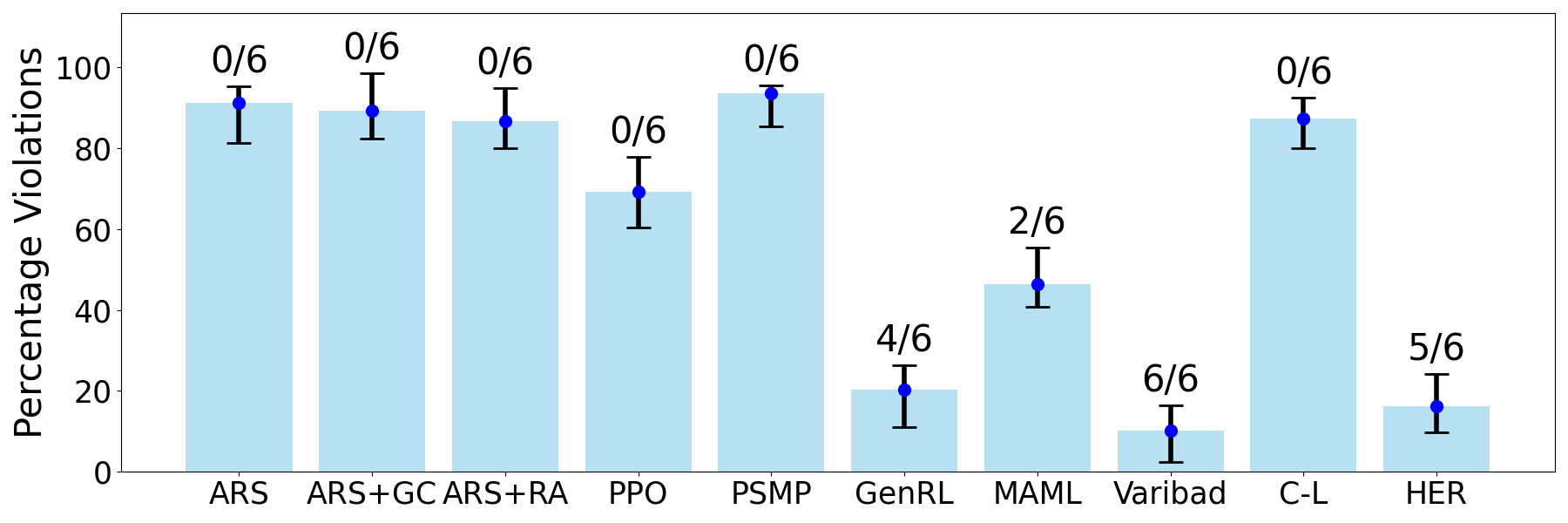}
        \caption{}
        \label{fig:reacher_robot}
    \end{subfigure}
    \caption{(a) Illustration of the task in the Reacher environment. (b) Violation plots for the Reacher environment.}
    \label{results:reacher:appendix}
\end{figure*}

\begin{figure*}[t]
    \centering
    
    \begin{subfigure}[t]{0.45\textwidth}
        \centering
        \includegraphics[width=\textwidth]{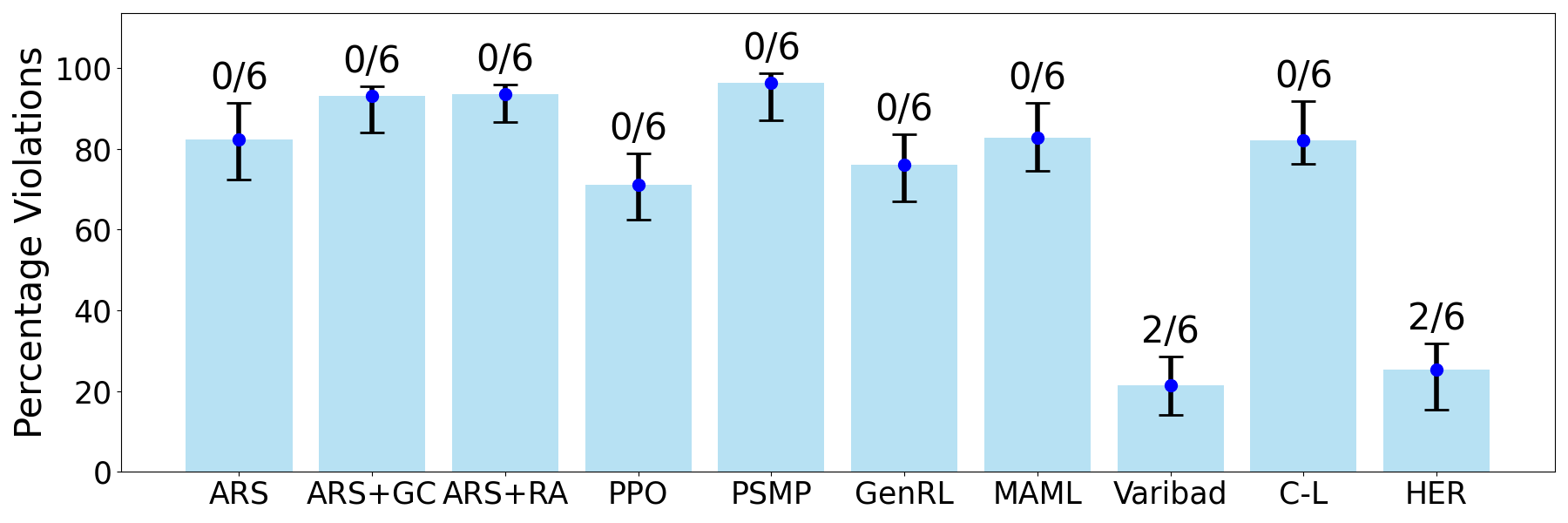}
        \caption{No Obstacle}
        \label{fig:fetch_reach}
    \end{subfigure}\hfill
    \begin{subfigure}[t]{0.45\textwidth}
        \centering
        \includegraphics[width=\textwidth]{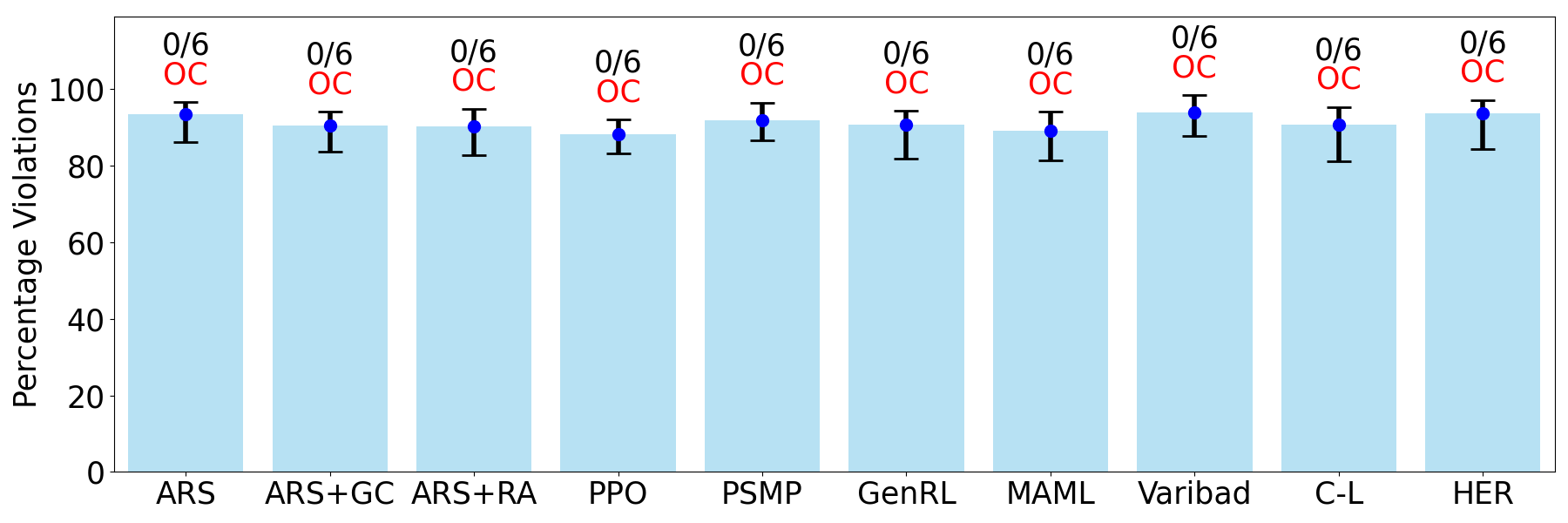}
        \caption{With Obstacle}
        \label{fig:fetch_reach_obs}
    \end{subfigure}
    
    \caption{{Violation plots for the Fetch-Reach environment.}}
    \label{fig:fetch_comparison_histograms}
\end{figure*}

\section{Experimental Evaluation}
\label{sec:experiments}

The purpose of the empirical evaluation is to demonstrate that our certification procedure can serve as a litmus test to examine the generalizability of RL algorithms. 

\subsection{Experimental Setup}

\noindent\textbf{Objective and Procedure.} Our objective is to demonstrate that a lower percentage of certificate violations indicates better generalizability of an RL algorithm. We do so by comparing the results of our certificate-based evaluation to ground truth.

Given an RL algorithm $\A$ and an inductive task $\T$, we learn a (generalizable) policy using the first $n$ tasks $\T_0, \T_1, \cdots \T_{n-1}$ of the inductive task. Then, we evaluate the following two metrics to assess generalizability of the learned policy:
\begin{itemize}[leftmargin=10pt]
    \item (Ground Truth) We compute the {\em number of unseen tasks solved} by the learned policy. That is, we consider how many tasks starting from $\T_{n}, \T_{n+1},\cdots $ the policy can solve. The larger the number of unseen tasks the policy can solve, the better its generalizability. 
    
    \item (Certificate-Guided) We use our framework in Section~\ref{sec:evaluation} to compute the percentage of certificate violations. The lower the percentage, the better its generalizability. 
\end{itemize}
By demonstrating a strong correlation between the two metrics, we will confirm the effectiveness of our framework in assessing generalization. 

To ensure fairness, all RL algorithms are trained on identical tasks from the
inductive family. In the ground truth evaluation, all algorithms are tested
on the same unseen tasks. For each inductive task, the certificate is trained
once and used to evaluate all algorithms.

\noindent\textbf{RL Algorithms.} We evaluate a diverse set of state-of-the-art generalizable RL algorithms that have the ability to generalize to similar but different tasks (for details, see Appendix~\ref{app:hyperparams}):
\begin{itemize}[leftmargin=*]
    \item Standard (Non-Generalizable) RL Algorithms, such as ARS and its variants (ARS + Reward Aggregation and ARS + Goal Conditioning) and PPO ~\citep{schulman2017proximal}, which serve as strong baselines despite lacking explicit generalization mechanisms. 
    \item Generalizable RL algorithms across three categories:
    \begin{itemize}[leftmargin=*]
        \item {\em Inductive Generalization algorithms} PSMP~\citep{inala2020synthesizing} and GenRL-ARS~\citep{SubramanianKRB24}, which leverage inductive structures between tasks; 
        \item {\em Meta-Learning algorithms} MAML-Reinforce~\citep{finn2017model} and VariBAD-A2C~\citep{zintgraf2021varibad}, enabling rapid adaptation through task-specific representations;  
        \item {\em Goal-Conditioned algorithms} C-Learning~\citep{naderian2021c} and HER-DDPG~\citep{andrychowicz2017hindsight}, which specialize in goal-conditioned generalization.
\end{itemize}
\end{itemize}

\noindent\textbf{Environments and Inductive Tasks.} 
In the interest of space, we discuss our results on three environments of increasing difficulty: 
1) \textbf{Car} environment~\citep{inala2020synthesizing}, which involves moving a rectangular car from an initial location to a goal location in a continuous 2D plane; 
2) \textbf{Reacher two-arm} environment, which involves controlling a planar two-link robotic arm to move its end effector to a target position; and 
3) \textbf{Fetch-Reach} environment~\citep{brockman2016openai}, which involves a 7-DOF robotic arm moving its end effector from an initial position to a goal position in a continuous 3D space. A more detailed description of the environments, task families, and experimental setup is provided in Appendix~\ref{app:expset}.

Illustrations of inductive reach-avoid tasks for the Car environment are given in Fig.~\ref{fig:r4illus} and Fig.~\ref{fig:side_by_side_figures}. The inductive reach-avoid task for the Fetch-Reach environment is to navigate from a moving initial location to a stationary goal location, similar to Fig.~\ref{fig:side_by_side_figures}. We also consider the {\em reachability } variant of all these tasks where the obstacle is absent. 
Full descriptions of environments, tasks, and results on more environments are presented in Appendices~\ref{app:training},~\ref{app:hyperparams} and~\ref{app:expset}.


\noindent\textbf{How to read {\em Violation Plots}. } The results are presented using {\em violation plots} that display the percentage of certificate violations and the number of successfully solved tasks for various RL algorithms (see Figs.~\ref{fig:car_comparison_histograms}, ~\ref{fig:fetch_comparison_histograms} and~\ref{results:reacher:appendix}). Each bar shows the mean percentage of violations over 10 runs with error bars indicating the variance. The number of solved tasks (e.g., $2/15$) is annotated above the bar, providing ground-truth evidence for the correctness of the certificate. This means that the algorithm generalized to 2 out of 15 unseen tasks. High violation rates correspond to poor generalization, as indicated by a low number of solved tasks, while low violation rates indicate better generalizability, as indicated by a high number of solved tasks. For reach-avoid tasks, "\textcolor{red}{OC}" (Obstacle Collision) is used to indicate when a trajectory fails by colliding with an obstacle, with all steps after the collision treated as violations.

\subsection{Observations and Results}

\noindent\textbf{Certificate violations are a good measure of generalizability.} All our experiments show that a lower percentage of certificate violations correlates with a higher number of successfully solved unseen tasks, hence indicating better generalizability properties of RL algorithms. Similarly, a higher percentage of certificate violations correlates with a lower number of successfully solved unseen tasks. 
Our evaluation results also align with priors one may have w.r.t. algorithm performance. For instance, standard (non-generalizable) RL approaches ARS and PPO perform poorly in all experiments.

\noindent\textbf{Reach-avoid tasks are hard generalization benchmarks.} Most RL algorithms performed poorly on reach-avoid tasks, as shown by high certificate violations and low solved task counts. In the high-dimensional obstacle environments such as Car and Fetch, even generalizable RL algorithms struggled, with some runs showing over 90\% violations. 
This illustrates that obstacles create a significant challenge for RL generalization. To analyze further, we tested in low-dimensional obstacle environments, such as Car2D, and observed similarly poor performance (Results are shown in Fig.~\ref{results:car2d:appendix} in the Appendix). This highlights that obstacles are inherently difficult constraints, even in simpler settings, underscoring their role as a strict constraint for generalization in locomotion-based navigation tasks. In contrast to reach-avoid tasks (tasks with obstacles), RL algorithms in environments without obstacles demonstrated better performance. Generalizable algorithms, such as GenRL and VariBAD, consistently achieved low violations. 
This underscores the relative simplicity of reachability tasks compared to reach-avoid tasks, where obstacle avoidance is not a factor and algorithms without explicit generalization capabilities perform poorly regardless of the task.

%% file: Section/conclusion.tex
\vspace{-0.5em}
\section{Conclusion}
\label{sec:conclusion}

This work presents a novel framework for evaluating and comparing the generalization capabilities of RL agents. Our framework is based on the proposed notion of certificates of correct generalization, which we use to provide a method for assessing RL agents' performance on unseen tasks. Our results demonstrate that minimizing certificate violations strongly correlates with better generalization, making this approach a reliable litmus test for evaluating RL generalizability. The experiments also confirm the effectiveness of our framework. 

\subsection*{Acknowledgment} Research reported in this publication was partly supported by an Amazon Research Award, Fall 2024

%% file: Appendix/main_appendix.tex
\section*{Appendix}

\input{Appendix/limitations.tex}
\input{Appendix/societal.tex}

\input{Appendix/certificate_algorithm}

\input{Appendix/exp_setup}
\input{Appendix/algorithms}
\input{Appendix/results}

%% file: Appendix/limitations.tex
\section{Limitations}\label{app:limitations}

The framework relies on the availability of high-quality demonstration trajectories to train the certificate function, which may not always be feasible in complex or poorly understood environments. It also assumes access to task-trajectory pairs that are representative of the broader task distribution, limiting its applicability in settings with high task variability or sparse feedback. 

%% file: Appendix/societal.tex
\section{Societal Impact}\label{app:soc}

The proposed framework enables more rigorous evaluation of reinforcement learning generalization, which is essential for deploying AI systems in real-world, safety-critical settings such as robotics and autonomous vehicles. By identifying failure modes and comparing generalization ability across agents, it can help improve reliability and trust in AI.

%% file: Appendix/certificate_algorithm.tex
\section{Training Details}\label{app:training}

\begin{algorithm*}[t]
\caption{Learning Reach-Avoid Certificates}
\label{alg:reachavoidcertificate_training}
\textbf{Input}: Inductive task family $\{\mathcal{T}_i\}_{i=0}^{N}$ with goal sets $\{G_i\}$ and initial state distributions $\{\eta_i\}$. \\
Training tasks: $\mathcal{T}_i$ for $i = 0, \dots, n-1$. \\
Testing tasks: $\mathcal{T}_i$ for $i = n, \dots, N$. \\
\textbf{Output}: Trained certificate function $\mathcal{C}(s, i)$. \\
\begin{algorithmic}[1]
\STATE Generate oracle trajectories $\zeta_i = (s_{i,0}, s_{i,1}, \dots, s_{i,l_i})$ for training tasks $\mathcal{T}_i$ ($i = 0, \dots, n$) to solve their respective goal sets $G_i$.
\STATE Initialize LSTM model $\mathcal{C}(s, i)$ with random weights.

\WHILE{the model has not converged}
\STATE Compute the loss:
\begin{align*}
&\mathsf{Loss} = \sum_{i=0}^{n-1} \sum_{(s_{i,j})_{j=0}^l \in \zeta_i} \Bigg( \\
    &+ \underbrace{\lambda_{\text{safe}} \sum_{j=0}^{l} 
    \max\{0, -C(s_{i,j}, i)\} \cdot 1_{s_{i,j} \in X_s}}_{\text{Penalizes to enforce condition (\ref{reach-avoid-def:nonneg})}} \\
    &+ \underbrace{\sum_{j=0}^{l - 1} 
    \max\{0, C(s_{i,j+1}, i) + \varepsilon_i - C(s_{i,j}, i)\} 
    \cdot 1_{(s_{i,j}, s_{i,j+1}) \in X_s}}_{\text{Penalizes to enforce condition (\ref{reach-avoid-def:sdwt})}} \\ 
    &+ \underbrace{%
    \begin{aligned}[t]
      \sum_{(s_{i+1,j})_{j=0}^l \in \zeta_{i+1}}& 
      \max\{0, C(s_{i+1,0}, i+1) + \varepsilon_i - C(s_{i,l}, i)\} \\
      &\cdot 1_{s_{i,l} \in G_i}
    \end{aligned}
    }_{\text{Penalizes to enforce condition (\ref{reach-avoid-def:sdat})}} \\
    &+ \underbrace{\lambda_{\text{unsafe}} \sum_{j=0}^{l} 
    \max\{0, C(s_{i,j}, i)\} \cdot 1_{s_{i,j} \in X_u}}_{\text{Penalizes to enforce condition (\ref{reach-avoid-def:neg})}} 
\Bigg)
\end{align*}
\STATE Perform backpropagation and update the weights of $\mathcal{C}(s, i)$.
\ENDWHILE
\STATE \textbf{return} trained certificate function $\mathcal{C}(s, i)$.
\end{algorithmic}
\end{algorithm*}

The experimental setup includes the following details:

\begin{itemize}
    \item \textbf{Number of Training Tasks:} 5 tasks for Car and Car2D environments, and 4 tasks for Reacher and Fetch environments.
    \item \textbf{Number of Unseen Tasks:} 15 tasks for Car and Car2D environments, and 6 tasks for Reacher and Fetch environments.
    \item \textbf{Number of Repetitions of the Experiments:} 10.
    \item \textbf{Training Compute:} SLURM cluster running Intel Xeon Gold 6226 CPUs, operating at 2.7 GHz with 24 cores per node. Each node is equipped with 192 GB of RAM.
    \item Our codebase is available at \href{https://anonymous.4open.science/r/certificates-B6D7}{https://anonymous.4open.science/r/certificates-B6D7}.
\end{itemize}

Learning hyperparameters for our framework are given in Tables~\ref{tab:hyp_lstm_reach} and~\ref{tab:hyp_lstm_avoid}.

\begin{table*}[!ht]
\centering
\begin{tabular}{|p{3cm}|p{3cm}|p{3cm}|p{3cm}|}
\hline
\textbf{Hyperparameter}         & \textbf{Car-Parking}   & \textbf{Reacher}   & \textbf{Fetch-Reach} \\ \hline
Learning Rate                   & \( 1 \times 10^{-4} \) & \( 1 \times 10^{-4} \) & \( 1 \times 10^{-4} \) \\ \hline
Batch Size                      & 64                     & 64                     & 128 \\ \hline
Sequence Length                 & 20                     & 20                     & 50 \\ \hline
Dropout Rate                    & 0.2                    & 0.2                    & 0.2 \\ \hline
Optimizer                       & Adam                   & Adam                   & Adam \\ \hline
Gradient Clipping               & \( \pm 0.5 \)          & \( \pm 0.5 \)          & \( \pm 0.5 \) \\ \hline
Network Architecture            & 2 hidden LSTM layers, 32 nodes (\texttt{tanh}); & 2 hidden LSTM layers, 32 nodes (\texttt{tanh}); & 3 hidden LSTM layers, 64 nodes (\texttt{tanh}); \\ 
                                & output layer (\texttt{ReLU})               & output layer (\texttt{ReLU})               & output layer (\texttt{ReLU}) \\ \hline

\end{tabular}
\caption{Hyperparameters for the LSTM-based Reachability Certificate Function}
\label{tab:hyp_lstm_reach}
\end{table*}

\begin{table*}[!ht]
\centering
\begin{tabular}{|p{3cm}|p{3cm}|p{3cm}|p{3cm}|}
\hline
\textbf{Hyperparameter}         & \textbf{Car-Parking}   & \textbf{Reacher}   & \textbf{Fetch-Reach} \\ \hline
Learning Rate                   & \( 1 \times 10^{-4} \) & \( 1 \times 10^{-4} \) & \( 1 \times 10^{-4} \) \\ \hline
Batch Size                      & 64                     & 64                     & 128 \\ \hline
Sequence Length                 & 20                     & 20                     & 50 \\ \hline
Dropout Rate                    & 0.2                    & 0.2                    & 0.2 \\ \hline
Optimizer                       & Adam                   & Adam                   & Adam \\ \hline
Gradient Clipping               & \( \pm 0.5 \)          & \( \pm 0.5 \)          & \( \pm 0.5 \) \\ \hline
Network Architecture            & 2 hidden LSTM layers, 32 nodes (\texttt{tanh}); & 2 hidden LSTM layers, 32 nodes (\texttt{tanh}); & 3 hidden  LSTM layers, 64 nodes (\texttt{tanh}); \\ 
                                & output layer                & output layer                & output layer \\
                                & (no activation) & (no activation) & (no activation) \\ \hline
\end{tabular}
\caption{Hyperparameters for the LSTM-based Reach-Avoid Certificate Function}
\label{tab:hyp_lstm_avoid}
\end{table*}

%% file: Appendix/exp_setup.tex
\section{Experimental Setup}\label{app:expset}
Our approach is evaluated across several environments, including the Car, Car2D, Reacher Two-Arm Robot, and Fetch-Reach settings, which feature diverse tasks and specifications. All tasks involve moving from an initial state to a target state where reachability is the main objective, with the environment subject to different dynamics and state spaces.

The \textbf{Car environment} has a 4-dimensional state space, including x and y positions, orientation, and velocity, with 2 actions: velocity and steering direction. The \textbf{Car2D environment} has a simplified 2-dimensional state space and observation space, consisting of the x and y coordinates of the agent, with 2 actions: velocity and direction. The \textbf{OpenAI Fetch-Reach environment} involves a robotic arm with an 8-dimensional state space, including the x, y, z positions of the end effector, state of the left and right gripper, and the x, y, z velocities of the end effector, with 3 actions: displacement along the x, y, and z axes. The \textbf{Reacher Two-Arm Robot environment} uses a 2-dimensional state space representing the x and y coordinates of the end effector, with 2 actions controlling the angles of the two joints.

The inductive tasks in these environments are systematically defined by updating the initial distributions and target goals according to the functions $ \text{update}\_{\text{init}} $ and $ \text{update}\_{\text{goal}} $, respectively:

\begin{itemize}
    \item \textbf{Car \& Car2D Environment} (Fig.~\ref{fig:motivating_example_cpmgiobs}): 
    \begin{itemize}
        \item Setting 1: the initial distribution is updated by shifting the initial position along the x-axis by a fixed value $ C $, i.e.,
        \[
        \eta_{i+1}(s) = \eta_i(s + (C, 0)).
        \]
        \[
        G_{i+1} = G_i + (0, 0).
        \]
        \item Setting 2: the goal is updated by shifting the target position along the x-axis by $ C $, i.e.,
        \[
        \eta_{i+1}(s) = \eta_i(s + (0, 0)).
        \]
        \[
        G_{i+1} = G_i + (C, 0).
        \]
        \item Setting 3: both the initial and target positions are shifted along the x-axis by $C$, combining both updates.

        \[
        \eta_{i+1}(s) = \eta_i(s + (C, 0)).
        \]
        \[
        G_{i+1} = G_i + (C, 0).
        \]
    \end{itemize}

   \item \textbf{Reacher two-arm robot environment} (Fig.~\ref{fig:reacher_illus}):  The goal is to move a box from height $ H - C $ in the \textit{Source} stack to height $ C + 1 $ in the \textit{Target} stack where H is the initial height of the source tower.

    \[
    \eta_{i+1}(s) = \eta_i(s + (0, H - C)).
    \]
    \[
    G_{i+1} = G_{i} + (0, C + 1).
    \]
    

    \item \textbf{Fetch-Reach Environment}: In this environment, the inductive tasks are defined by moving the initial distribution along the $y$-axis by a fixed increment $C$, while keeping the goal region fixed, i.e.,
\[
    \eta_{i+1}(s) = \eta_i(s + (0, C, 0)).
\]
\[
    G_{i+1} = G_i.
\]
Thus, each successive task starts from a shifted initial position, but the target goal remains stationary. 
\end{itemize}


%% file: Appendix/algorithms.tex
\section{Algorithms and Hyperparameters}
\label{app:hyperparams}

The evaluated algorithms are categorized as follows:

\begin{itemize}[leftmargin=15pt]
    \item \textbf{Standard RL Algorithms:}  
    \textbf{ARS}, \textbf{ARS+GC}, \textbf{ARS+RA}, and \textbf{PPO} are baseline reinforcement learning methods known for their effectiveness in diverse tasks. These algorithms lack explicit mechanisms for generalization, focusing instead on optimizing performance in specific tasks.

    \item \textbf{Inductive Generalization Algorithms:}  
    \textbf{PSMP} and \textbf{GenRL} leverage inductive biases to generalize across tasks with structural similarities. These methods enhance transferability by exploiting shared properties within task families.

    \item \textbf{Meta-Learning Algorithms:}  
    \textbf{MAML} and \textbf{Varibad} are designed for rapid adaptation to new tasks with minimal training. \textbf{MAML} employs gradient-based meta-learning, while \textbf{Varibad} uses variational inference for task embeddings, enabling zero-shot generalization.

    \item \textbf{Goal-Conditioned RL Algorithms:}  
    \textbf{C-Learning} and \textbf{HER} explicitly incorporate goals into policy learning, optimizing for tasks requiring specific goal states. These algorithms are well-suited for environments with dynamic or goal-driven objectives.
\end{itemize}

Hyperparameters used for these algorithms are given in Tables~\ref{tab:hyper_ars} to~\ref{tab:hyper_her}

\begin{table*}[!ht]
\centering
\begin{tabular}{|p{3cm}|p{3cm}|p{3cm}|p{3cm}|}
\hline
\textbf{Hyperparameter}         & \textbf{Car \& Car2D}   & \textbf{Reacher}   & \textbf{Fetch-Reach} \\ \hline
Episode Length               & 30                     & 30                     & 100 \\ \hline
Learning Rate                   & \( 1 \times 10^{-4} \) & \( 1 \times 10^{-4} \) & \( 1 \times 10^{-4} \) \\ \hline
Number of Directions Sampled    & 30                     & 30                     & 30 \\ \hline
Number of Top Samples Used      & 8                      & 8                      & 8 \\ \hline
Network Architecture            & 2 hidden layers, 32 nodes (\texttt{ReLU}), & 2 hidden layers, 32 nodes (\texttt{ReLU}), & 3 hidden layers, 64 nodes (\texttt{ReLU}), \\ 
                                & output layer (\texttt{tanh})               & output layer (\texttt{tanh})               & output layer (\texttt{tanh}) \\ \hline
\end{tabular}
\caption{Hyperparameters for ARS, ARS + GC, ARS + RA}
\label{tab:hyper_ars}
\end{table*}

\begin{table*}[!ht]
\centering
\begin{tabular}{|p{3cm}|p{3cm}|p{3cm}|p{3cm}|}
\hline
\textbf{Hyperparameter}         & \textbf{Car \& Car2D}   & \textbf{Reacher}   & \textbf{Fetch-Reach} \\ \hline
Episode Length               & 30                     & 30                     & 100 \\ \hline
Learning Rate                   & \( 1 \times 10^{-4} \) & \( 1 \times 10^{-4} \) & \( 1 \times 10^{-4} \) \\ \hline
Batch Size                      & 64                     & 64                     & 256 \\ \hline
Discount Factor   & 0.99                   & 0.99                   & 0.99 \\ \hline
GAE Parameter    & 0.95                   & 0.95                   & 0.95 \\ \hline
Clipping Parameter  & 0.2               & 0.2                   & 0.2 \\ \hline
Value Loss Coefficient  & 0.5                & 0.5                   & 0.5 \\ \hline
Entropy Coefficient  & 0.01                  & 0.01                  & 0.01 \\ \hline
Replay Buffer Size              & \( 10^4 \)            & \( 10^4 \)            & \( 10^4 \) \\ \hline
Network Architecture            & 2 hidden layers, 32 nodes (\texttt{ReLU}), & 2 hidden layers, 32 nodes (\texttt{ReLU}), & 3 hidden layers, 64 nodes (\texttt{ReLU}), \\ 
                                & output layer (\texttt{tanh})               & output layer (\texttt{tanh})               & output layer (\texttt{tanh}) \\ \hline
\end{tabular}
\caption{Hyperparameters for PPO}
\label{tab:hyper_ppo}
\end{table*}

\begin{table*}[!ht]
\centering
\begin{tabular}{|p{3cm}|p{3cm}|p{3cm}|p{3cm}|}
\hline
\textbf{Hyperparameter}         & \textbf{Car \& Car2D}   & \textbf{Reacher}   & \textbf{Fetch-Reach} \\ \hline
Episode Length               & 30                     & 30                     & 100 \\ \hline
Learning Rate                   & \( 1 \times 10^{-4} \) & \( 1 \times 10^{-4} \) & \( 1 \times 10^{-4} \) \\ \hline
Batch Size                      & 64                     & 64                     & 256 \\ \hline
Maximum Number of Modes         & 2                      & 2                      & 2 \\ \hline
\( \lambda \)                   & 100                    & 100                    & 100 \\ \hline
Student Action Function Grammar & Proportional           & Proportional           & Proportional \\ \hline
Switching Condition Grammar     & Boolean tree (depth 1 or 2) & Boolean tree (depth 1 or 2) & Boolean tree (depth 1 or 2) \\ \hline
Sampling for Teacher Initialization & 10 top trajectories (CEM) & 10 top trajectories (CEM) & 10 top trajectories (CEM) \\ \hline
Network Architecture            & 2 hidden layers, 32 nodes (\texttt{ReLU}), & 2 hidden layers, 32 nodes (\texttt{ReLU}), & 3 hidden layers, 64 nodes (\texttt{ReLU}), \\ 
                                & output layer (\texttt{tanh})               & output layer (\texttt{tanh})               & output layer (\texttt{tanh}) \\ \hline
\end{tabular}
\caption{Hyperparameters for PSMP}
\label{tab:hyper_PSMP}
\end{table*}

\begin{table*}[!ht]
\centering
\begin{tabular}{|p{3cm}|p{3cm}|p{3cm}|p{3cm}|}
\hline
\textbf{Hyperparameter}         & \textbf{Car \& Car2D}   & \textbf{Reacher}   & \textbf{Fetch-Reach} \\ \hline
Learning Rate                   & \( 1 \times 10^{-4} \) & \( 1 \times 10^{-4} \) & \( 1 \times 10^{-4} \) \\ \hline
Number of Directions Sampled    & 30                     & 30                     & 30 \\ \hline
Number of Top Samples Used      & 8                      & 8                      & 8 \\ \hline
Reward Aggregation Scheme       & softmin                & softmin                & softmin \\ \hline
Batch Size                      & 64                     & 64                     & 256 \\ \hline
Network Architecture            & 2 hidden layers, 32 nodes (\texttt{ReLU}), & 2 hidden layers, 32 nodes (\texttt{ReLU}), & 3 hidden layers, 64 nodes (\texttt{ReLU}), \\ 
                                & output layer (\texttt{tanh})               & output layer (\texttt{tanh})               & output layer (\texttt{tanh}) \\ \hline
\end{tabular}
\caption{Hyperparameters for GenRL-ARS}
\label{tab:hyper_genrl}
\end{table*}

\begin{table*}[!ht]
\centering
\begin{tabular}{|p{3cm}|p{3cm}|p{3cm}|p{3cm}|}
\hline
\textbf{Hyperparameter}         & \textbf{Car \& Car2D}   & \textbf{Reacher}   & \textbf{Fetch-Reach} \\ \hline
Episode Length               & 30                     & 30                     & 100 \\ \hline
Learning Rate                   & \( 1 \times 10^{-4} \) & \( 1 \times 10^{-4} \) & \( 1 \times 10^{-4} \) \\ \hline
Batch Size                      & 64                     & 64                     & 256 \\ \hline
Number of Meta-Tasks            & 4                      & 5                      & 5 \\ \hline
Number of Inner Loop Steps      & 0                      & 0                      & 0 \\ \hline
Discount Factor   & 0.99                   & 0.99                   & 0.99 \\ \hline
Entropy Coefficient             & 0.01                   & 0.01                   & 0.01 \\ \hline
Policy Network Architecture     & 2 hidden layers, 32 nodes (\texttt{ReLU}), & 2 hidden layers, 32 nodes (\texttt{ReLU}), & 3 hidden layers, 64 nodes (\texttt{ReLU}), \\ 
                                & output layer (\texttt{tanh})               & output layer (\texttt{tanh})               & output layer (\texttt{tanh}) \\ \hline
Optimizer                       & Adam                   & Adam                   & Adam \\ \hline
Gradient Clipping               & \( \pm 0.5 \)          & \( \pm 0.5 \)          & \( \pm 0.5 \) \\ \hline
KL-Divergence Regularization    & Enabled, coefficient = 0.1 & Enabled, coefficient = 0.1 & Enabled, coefficient = 0.1 \\ \hline
Trajectory Length               & 200 steps              & 200 steps              & 200 steps \\ \hline
Meta-Batch Size                 & 4 tasks               & 5 tasks               & 5 tasks \\ \hline
\end{tabular}
\caption{Hyperparameters for MAML-Reinforce}
\label{tab:hyper_maml}
\end{table*}

\begin{table*}[!ht]
\centering
\begin{tabular}{|p{3cm}|p{3cm}|p{3cm}|p{3cm}|}
\hline
\textbf{Hyperparameter}         & \textbf{Car \& Car2D}   & \textbf{Reacher}   & \textbf{Fetch-Reach} \\ \hline
Episode Length               & 30                     & 30                     & 100 \\ \hline
Learning Rate                   & \( 1 \times 10^{-4} \) & \( 1 \times 10^{-4} \) & \( 1 \times 10^{-4} \) \\ \hline
Batch Size                      & 64                     & 64                     & 256 \\ \hline
Epsilon                         & \( 1 \times 10^{-5} \)  & \( 1 \times 10^{-5} \)  & \( 1 \times 10^{-5} \) \\ \hline
Discount Factor   & 0.95                   & 0.95                   & 0.95 \\ \hline
Max Gradient Norm               & 0.5                    & 0.5                    & 0.5 \\ \hline
Value Loss Coefficient          & 0.5                    & 0.5                    & 0.5 \\ \hline
Entropy Coefficient             & 0.01                   & 0.01                   & 0.01 \\ \hline
GAE Parameter       & 0.95                   & 0.95                   & 0.95 \\ \hline
ELBO Loss Coefficient           & 1.0                    & 1.0                    & 1.0 \\ \hline
Task Embedding Size             & 5                      & 5                      & 5 \\ \hline
Network Architecture            & 2 hidden layers, 32 nodes (\texttt{ReLU}), & 2 hidden layers, 32 nodes (\texttt{ReLU}), & 3 hidden layers, 64 nodes (\texttt{ReLU}), \\ 
                                & output layer (\texttt{tanh})               & output layer (\texttt{tanh})               & output layer (\texttt{tanh}) \\ \hline
Encoder Architecture            & Fully connected (40 nodes), GRU (hidden size 64), & Fully connected (40 nodes), GRU (hidden size 64), & Fully connected (40 nodes), GRU (hidden size 64), \\ 
                                & output layer (10 outputs, \texttt{ReLU})   & output layer (10 outputs, \texttt{ReLU})   & output layer (10 outputs, \texttt{ReLU}) \\ \hline
Reward Decoder Architecture     & 2 hidden layers, 32 nodes (\texttt{ReLU}), & 2 hidden layers, 32 nodes (\texttt{ReLU}), & 3 hidden layers, 64 nodes (\texttt{ReLU}), \\ 
                                & output layer (\texttt{tanh})               & output layer (\texttt{tanh})               & output layer (\texttt{tanh}) \\ \hline
Decoder Loss Function           & Binary Cross Entropy   & Binary Cross Entropy   & Binary Cross Entropy \\ \hline
\end{tabular}
\caption{Hyperparameters for Varibad-A2C}
\label{tab:hyper_varibad}
\end{table*}

\begin{table*}[!ht]
\centering
\begin{tabular}{|p{3cm}|p{3cm}|p{3cm}|p{3cm}|}
\hline
\textbf{Hyperparameter}         & \textbf{Car \& Car2D}   & \textbf{Reacher}   & \textbf{Fetch-Reach} \\ \hline
Episode Length               & 30                     & 30                     & 100 \\ \hline
Learning Rate                   & \( 1 \times 10^{-4} \) & \( 1 \times 10^{-4} \) & \( 1 \times 10^{-4} \) \\ \hline
Batch Size                      & 64                     & 64                     & 256 \\ \hline
Gradient Steps per Episode      & 64                     & 64                     & 64 \\ \hline

Network Architecture            & 2 hidden layers, 32 nodes (\texttt{ReLU}), & 2 hidden layers, 32 nodes (\texttt{ReLU}), & 3 hidden layers, 64 nodes (\texttt{ReLU}), \\ 
                                & output layer (\texttt{tanh})               & output layer (\texttt{tanh})               & output layer (\texttt{tanh}) \\ \hline
Target Network Update Frequency & Every 10 steps         & Every 10 steps         & Every 10 steps \\ \hline
Discount Factor   & 0.9                    & 0.9                    & 0.9 \\ \hline
\( \kappa \)                    & 3                      & 3                      & 3 \\ \hline
\end{tabular}
\caption{Hyperparameters for C-Learning}
\label{tab:hyper_cl}
\end{table*}

\begin{table*}[!ht]
\centering
\begin{tabular}{|p{3cm}|p{3cm}|p{3cm}|p{3cm}|}
\hline
\textbf{Hyperparameter}         & \textbf{Car \& Car2D}   & \textbf{Reacher}   & \textbf{Fetch-Reach} \\ \hline
Episode Length               & 30                     & 30                     & 100 \\ \hline
Learning Rate                   & \( 1 \times 10^{-4} \) & \( 1 \times 10^{-4} \) & \( 1 \times 10^{-4} \) \\ \hline
Critic Learning Rate            & \( 1 \times 10^{-3} \) & \( 1 \times 10^{-3} \) & \( 1 \times 10^{-3} \) \\ \hline
Batch Size                      & 64                     & 64                     & 256 \\ \hline
Discount Factor   & 0.98                   & 0.98                   & 0.98 \\ \hline
Replay Buffer Size              & \( 10^4 \)            & \( 10^4 \)            & \( 10^4 \) \\ \hline
Policy Noise                    & 0.2                   & 0.2                   & 0.2 \\ \hline
Target Network Update Rate  & 0.005       & 0.005                 & 0.005 \\ \hline
Hindsight Strategy              & Future strategy (relabel goals) & Future strategy (relabel goals) & Future strategy (relabel goals) \\ \hline
Hindsight Fraction              & 0.8                   & 0.8                   & 0.8 \\ \hline
Network Architecture            & 2 hidden layers, 32 nodes (\texttt{ReLU}), & 2 hidden layers, 32 nodes (\texttt{ReLU}), & 3 hidden layers, 64 nodes (\texttt{ReLU}), \\ 
                                & output layer (\texttt{tanh})               & output layer (\texttt{tanh})               & output layer (\texttt{tanh}) \\ \hline
\end{tabular}
\caption{Hyperparameters for HER-DDPG}
\label{tab:hyper_her}
\end{table*}

%% file: Appendix/results.tex



\begin{figure*}[ht]
    \centering
    
    \begin{subfigure}[b]{0.49\textwidth}
        \centering
        \includegraphics[width=\textwidth]{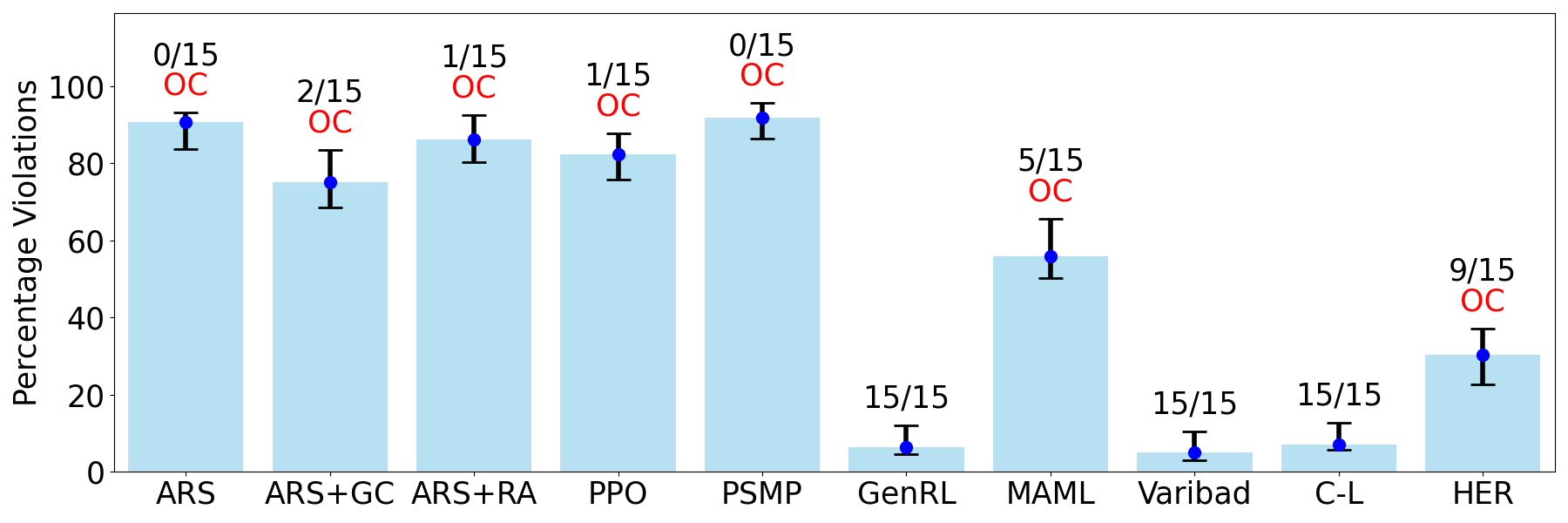}
        \caption{Car2D: Moving Initial Point with obstacle}
        \label{fig:initial_point_obs2d}
    \end{subfigure}
    \hfill
    \begin{subfigure}[b]{0.49\textwidth}
        \centering
        \includegraphics[width=\textwidth]{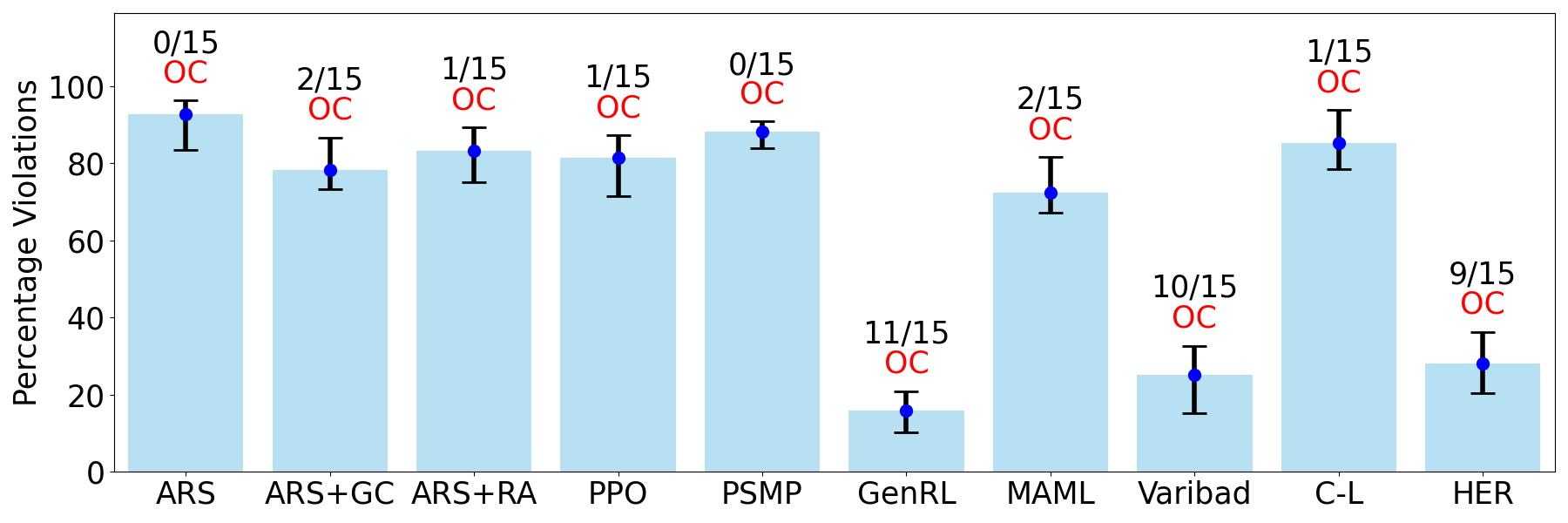}
        \caption{Car2D:  Moving Goal Point with obstacle}
        \label{fig:goal_point_obs2d}
    \end{subfigure}

    \begin{subfigure}[b]{0.49\textwidth}
        \centering
        \includegraphics[width=\textwidth]{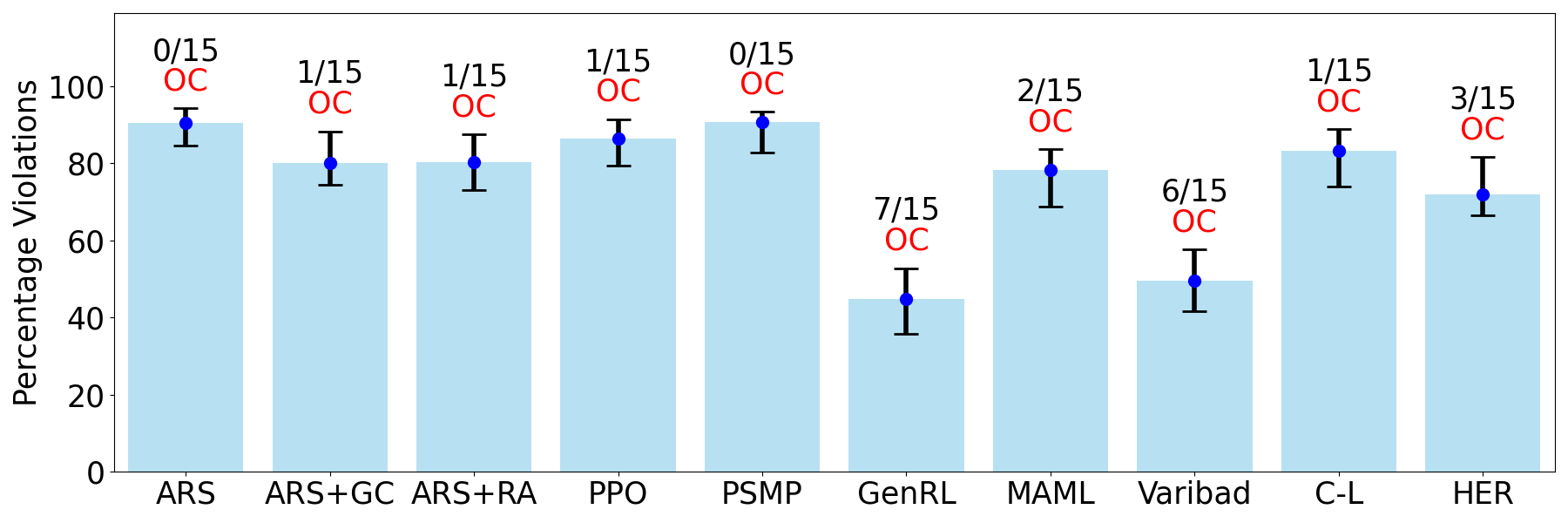}
        \caption{Car2D: Moving Initial and Goal Point with obstacle}
        \label{fig:initialandgoal_point_obs2d}
    \end{subfigure}

    \caption{Performance evaluation of tasks on the simpler Car2D environment. (a)-(c)  present the evaluation results, showing the percentage of certificate violations and the number of successfully solved tasks for various RL algorithms. Each marker represents the mean percentage of violations over 10 runs, with error bars indicating the range (minimum to maximum). The numbers annotated next to each marker (e.g., $0/15$) denote the number of successfully solved tasks out of the total test tasks. Higher violations correspond to poor generalization, while a lower percentage of violations indicates better generalization.}
    \label{results:car2d:appendix}

\end{figure*}

%% file: references.bib
@inproceedings{MalikLR21,
  author    = {Dhruv Malik and Yuanzhi Li and Pradeep Ravikumar},
  title     = {When Is Generalizable Reinforcement Learning Tractable?},
  booktitle = {Advances in Neural Information Processing Systems (NeurIPS)},
  year      = {2021}
}

@article{KirkZGR23,
  author  = {Robert Kirk and Amy Zhang and Edward Grefenstette and Tim Rockt{\"a}schel},
  title   = {A Survey of Zero-shot Generalisation in Deep Reinforcement Learning},
  journal = {Journal of Artificial Intelligence Research},
  volume  = {76},
  year    = {2023},
  pages   = {201--264}
}

@article{Korkmaz24,
  author  = {Ezgi Korkmaz},
  title   = {A Survey Analyzing Generalization in Deep Reinforcement Learning},
  journal = {arXiv preprint arXiv:2401.02349},
  year    = {2024}
}

@inproceedings{ZikelicLHC23,
  author    = {Dorde Zikelic and Mathias Lechner and Thomas A. Henzinger and Krishnendu Chatterjee},
  title     = {Learning Control Policies for Stochastic Systems with Reach-Avoid Guarantees},
  booktitle = {AAAI},
  year      = {2023},
  pages     = {11926--11935}
}

@article{DawsonGF23,
  author  = {Charles Dawson and Sicun Gao and Chuchu Fan},
  title   = {Safe Control With Learned Certificates: A Survey of Neural Lyapunov,
             Barrier, and Contraction Methods for Robotics and Control},
  journal = {IEEE Transactions on Robotics},
  volume  = {39},
  number  = {3},
  pages   = {1749--1767},
  year    = {2023}
}

@inproceedings{ChangRG19,
  author    = {Ya{-}Chien Chang and Nima Roohi and Sicun Gao},
  title     = {Neural Lyapunov Control},
  booktitle = {NeurIPS},
  year      = {2019},
  pages     = {3240--3249}
}

@article{AbateAGP21,
  author  = {Alessandro Abate and Daniele Ahmed and Mirco Giacobbe and Andrea Peruffo},
  title   = {Formal Synthesis of Lyapunov Neural Networks},
  journal = {IEEE Control Systems Letters},
  volume  = {5},
  number  = {3},
  pages   = {773--778},
  year    = {2021}
}

@inproceedings{EdwardsPA24,
  author    = {Alec Edwards and Andrea Peruffo and Alessandro Abate},
  title     = {Fossil 2.0: Formal Certificate Synthesis for the Verification and Control
               of Dynamical Models},
  booktitle = {ACM International Conference on Hybrid Systems: Computation and Control (HSCC)},
  year      = {2024},
  pages     = {26:1--26:10}
}

@inproceedings{Zhang0V023,
  author    = {Hongchao Zhang and Junlin Wu and Yevgeniy Vorobeychik and Andrew Clark},
  title     = {Exact Verification of ReLU Neural Control Barrier Functions},
  booktitle = {NeurIPS},
  year      = {2023}
}

@inproceedings{LechnerZCH22,
  author    = {Mathias Lechner and Dorde Zikelic and Krishnendu Chatterjee and
               Thomas A. Henzinger},
  title     = {Stability Verification in Stochastic Control Systems via Neural Network
               Supermartingales},
  booktitle = {AAAI},
  year      = {2022},
  pages     = {7326--7336}
}

@article{MathiesenCL23,
  author  = {Frederik Baymler Mathiesen and Simeon C. Calvert and Luca Laurenti},
  title   = {Safety Certification for Stochastic Systems via Neural Barrier Functions},
  journal = {IEEE Control Systems Letters},
  volume  = {7},
  pages   = {973--978},
  year    = {2023}
}

@inproceedings{ChatterjeeHLZ23,
  author    = {Krishnendu Chatterjee and Thomas A. Henzinger and
               Mathias Lechner and Dorde Zikelic},
  title     = {A Learner-Verifier Framework for Neural Network Controllers
               and Certificates of Stochastic Systems},
  booktitle = {TACAS},
  year      = {2023},
  pages     = {3--25}
}

@inproceedings{SubramanianKRB24,
 title={Inductive generalization in reinforcement learning from specifications},
  author={Subramanian, Vignesh and Kushwah, Rohit and Roy, Subhajit and Bansal, Suguman},
  booktitle={International Symposium on Automated Technology for Verification and Analysis},
  pages={277--298},
  year={2025},
  organization={Springer}
}

@article{schulman2017proximal,
  author  = {John Schulman and Filip Wolski and Prafulla Dhariwal and
             Alec Radford and Oleg Klimov},
  title   = {Proximal policy optimization algorithms},
  journal = {arXiv preprint arXiv:1707.06347},
  year    = {2017}
}

@article{brockman2016openai,
  author  = {Greg Brockman and Vicki Cheung and Ludwig Pettersson and
             Jonas Schneider and John Schulman and Jie Tang and Wojciech Zaremba},
  title   = {OpenAI Gym},
  journal = {arXiv preprint arXiv:1606.01540},
  year    = {2016}
}

@article{zintgraf2021varibad,
  author  = {Luisa Zintgraf and Sebastian Schulze and Cong Lu and
             Leo Feng and Maximilian Igl and Kyriacos Shiarlis and
             Yarin Gal and Katja Hofmann and Shimon Whiteson},
  title   = {VariBAD: Variational Bayes-adaptive deep RL via meta-learning},
  journal = {Journal of Machine Learning Research},
  volume  = {22},
  number  = {289},
  pages   = {1--39},
  year    = {2021}
}

@article{nagabandi2019deep,
  author  = {Anusha Nagabandi and Chelsea Finn and Sergey Levine},
  title   = {Deep online learning via meta-learning: continual adaptation for model-based RL},
  journal = {International Conference on Learning Representations},
  year    = {2019}
}

@inproceedings{inala2020synthesizing,
  author    = {Jeevana Priya Inala and Osbert Bastani and Zenna Tavares and
               Armando Solar-Lezama},
  title     = {Synthesizing Programmatic Policies That Inductively Generalize},
  booktitle = {International Conference on Learning Representations},
  year      = {2020}
}

@inproceedings{finn2017model,
  author    = {Chelsea Finn and Pieter Abbeel and Sergey Levine},
  title     = {Model-Agnostic Meta-Learning for Fast Adaptation of Deep Networks},
  booktitle = {International Conference on Machine Learning (ICML)},
  pages     = {1126--1135},
  year      = {2017}
}

@article{naderian2021c,
  author  = {Panteha Naderian and Gabriel Loaiza-Ganem and Harry J. Braviner and
             Anthony L. Caterini and Jesse C. Cresswell and Tong Li and Animesh Garg},
  title   = {C-learning: Horizon-aware cumulative accessibility estimation},
  journal = {International Conference on Learning Representations},
  year    = {2021}
}

@article{hochreiter1997long,
  author  = {Sepp Hochreiter},
  title   = {Long Short-term Memory},
  journal = {Neural Computation},
  year    = {1997}
}

@article{beck2023survey,
  author  = {Jacob Beck and Risto Vuorio and Evan Zheran Liu and
             Zheng Xiong and Luisa Zintgraf and Chelsea Finn and Shimon Whiteson},
  title   = {A survey of meta-reinforcement learning},
  journal = {arXiv preprint arXiv:2301.08028},
  year    = {2023}
}

@article{andrychowicz2017hindsight,
  author  = {Marcin Andrychowicz and Filip Wolski and Alex Ray and
             Jonas Schneider and Rachel Fong and Peter Welinder and
             Bob McGrew and Josh Tobin and Pieter Abbeel and Wojciech Zaremba},
  title   = {Hindsight Experience Replay},
  journal = {Advances in Neural Information Processing Systems},
  volume  = {30},
  year    = {2017}
}
